\documentclass{article}

 \usepackage[preprint]{neurips_2026}

\usepackage[utf8]{inputenc} 
\usepackage[T1]{fontenc}    
\usepackage{hyperref}       
\usepackage{url}            
\usepackage{booktabs}       
\usepackage{amsfonts}       
\usepackage{nicefrac}       
\usepackage{microtype}      
\usepackage{xcolor}         

\usepackage[american]{babel}

\usepackage{mathtools} 
\usepackage{booktabs} 
\usepackage{tikz} 

\usepackage{algorithm}
\usepackage{algorithmic}

\usepackage{microtype}
\usepackage{graphicx}
\usepackage{booktabs}

\usepackage{amsthm}
\usepackage{thmtools,thm-restate}

\usepackage{hyperref}
\usepackage{caption}
\usepackage{subcaption}
\usepackage{bbm}

\usepackage{amsmath}
\usepackage{amssymb}
\usepackage{mathtools}

\usepackage{mwe}
\usepackage{tikz} 
\usepackage{pgf}
\usepackage{bbm}
\usepackage{extarrows} 

\usepackage[capitalize,noabbrev]{cleveref}

\usepackage{enumitem}

\newcommand{\ocal}[0]{\mathcal{O}}

\newcommand{\hcal}[0]{\mathcal{H}}
\newcommand{\E}[0]{\mathbbm{E}}

\newcommand{\pr}[0]{\mathbb{P}}

\newcommand{\ebb}[0]{\mathbb{E}}

\newcommand{\pmin}[0]{p_{\text{min}}}
\newcommand{\ahatj}[0]{\hat{a}_T(j)}
\newcommand{\astarj}[0]{a^*(j)}
\newcommand{\srtpinup}[0]{\text{SR} \paran{\pi, T, \bnu, \mathbf p}}
\newcommand{\srtpi}[0]{\text{SR}(\pi, T, \mathbf p)}
\newcommand{\algsr}[0]{\mathcal{A}_{\mathrm{SR}}}
\newcommand{\qpip}[0]{\mathbf q(\pi,\mathbf p)}
\newcommand{\qjpip}[0]{ q(\pi,\mathbf p)_j}
\newcommand{\Piind}[1]{\Pi_{\mathrm{#1}}}
\newcommand{\qstalphap}{\mathbf q^*_{\alpha}(\mathbf p)}
\newcommand{\qstalphapj}{ q^*_{\alpha}(\mathbf p)_j}
\newcommand{\vstalphap}{v^*_{\alpha}(\mathbf p)}
\newcommand{\alphmin}{\alpha_{\text{min}}(\mathbf p)}
\newcommand{\alphmininp}[1]{\alpha_{\text{min}}(#1)}

\newcommand{\paran}[1]{\left( #1 \right)}
\newcommand{\brak}[1]{\left[ #1 \right]}

\newcommand{\mb}[1]{\mathbf{#1}}
\newcommand{\mc}[1]{\mathcal{#1}}

\newcommand{\kl}[1]{\mathrm{KL}\left( #1 \right)}

\newcommand{\bnu}[0]{\boldsymbol{\nu}}
\newcommand{\bmu}[0]{\boldsymbol{\mu}}

\newcommand{\pht}[0]{\hat{\mb{p}}(t)}

\newcommand{\phtau}[1]{\hat{{\mb{p}}}(\tau_{#1})}
\newcommand{\phfreeze}{\hat{\mb{p}}^{\mathrm{fr}}}
\newcommand{\phfreezej}{\hat{{p}}^{\mathrm{fr}}_j}
\newcommand{\ntj}[0]{N(t)_j}

\newcommand{\deln}[1]{\Delta ^ {#1 - 1}}
\newcommand{\delk}[0]{\Delta ^ {k - 1}}
\newcommand{\delkplus}[0]{\Delta_+ ^ {k - 1}}

\newcommand{\norm}[1]{\left\lVert#1\right\rVert}
\newcommand{\lpnorm}[2]{\norm{#1}_{#2}}

\DeclareMathOperator*{\argmax}{arg\,max}
\DeclareMathOperator*{\argmin}{arg\,min}

\theoremstyle{plain}
\newtheorem{theorem}{Theorem}[section]
\newtheorem{proposition}[theorem]{Proposition}
\newtheorem{lemma}[theorem]{Lemma}

\theoremstyle{definition}

\theoremstyle{remark}

\usepackage[textsize=tiny]{todonotes}

\title{Active Context Selection Improves \\ 
Simple Regret in Contextual Bandits}

\author{
    Mohammad Shahverdikondori  \\
    College of Management of Technology, EPFL \\
    \texttt{mohammad.shahverdikondori@epfl.ch} \\
    \And
    Jalal Etesami \\
    Department of Computer Science, TU Munich\\
    \texttt{j.etesami@tum.de} \\
    \And
    Negar Kiyavash \\
    College of Management of Technology, EPFL \\
    \texttt{negar.kiyavash@epfl.ch} \\
}

\begin{document}

\maketitle

\begin{abstract}

We study the contextual multi-armed bandit problem with a finite context space (a.k.a. subpopulations), 
where the learner recommends a best action for each context and is evaluated by context-weighted simple regret.
Our guarantees are worst-case over the reward distributions, while remaining instance-dependent with respect to the context distribution vector $\mathbf p$. Akin to experimental design problems where the population of interest is fixed but the sampled subpopulation can be controlled, we allow the learner to actively choose which context to sample from.
For a known $\mathbf p$, we characterize tight regret rates: passive sampling where contexts are randomly revealed achieves regret of order 
{\footnotesize$\sqrt{n/T\lpnorm{\mb{p}}{1/2}}$}, 
whereas active sampling with allocation 
{\small$q_j \propto p_j^{2/3}$}
achieves the tight rate
{\footnotesize$\sqrt{n/T}\,\lpnorm{\mb{p}}{2/3}$}. 
The resulting improvement can be as large as 
{\small$\Theta(k^{1/4})$}, 
where $k$ is the number of contexts. 
We further extend the analysis to budgeted active sampling, characterize the corresponding tight rate, and identify when a limited active budget suffices to recover the fully active rate.
When $\mathbf p$ is unknown, we propose the Explore-Explore-Then-Commit (EETC) algorithm, which optimally balances estimating the context distribution and the time to switch to active allocation, such that for large horizons, it matches the known-$\mathbf p$ active rate up to constants. 
Experiments on synthetic and real-world data support our theoretical findings.
  
\end{abstract}

\section{Introduction} \label{sec: intro}

Sequential decision making problems often involve contextual information: the outcome distribution of an action may depend on observable characteristics of the unit under consideration \cite{bandit-book1-lattimore2020bandit, bandit-book2-bubeck2012regret}. In the finite context regime, each context value naturally defines a subpopulation of the target population, such as an age bracket, geographic region, or disease subtype. This viewpoint is especially common in experimental design applications, where the effect of an action, or treatment, may vary substantially across subpopulations. The learning objective is then not simply to identify a single globally best action, but to recommend the best treatment for each subpopulation \cite{perchet2013multi, russac2021b, curth2023adaptive}

This motivates the context-weighted simple regret objective studied in this paper: after a fixed experimental budget, the learner recommends one treatment per subpopulation. The resulting regret is the average recommendation gap weighted by the population distribution over subpopulations. Even when this distribution is known from historical data, the learner still faces the question of how to allocate the experimental budget across subpopulations. 
In the standard passive setting, contexts arrive randomly according to the population distribution and the learner only chooses the treatment \cite{krishnamurthy2023proportional, jorke2022simple, contextual-simple-regret-deshmukh2018simple}. 
However, in many applications, the experimenter may have partial control over which subpopulation to sample. For instance, an online platform may target specific user segments, or a clinical study may recruit participants from selected demographic groups, possibly at different costs. 
This raises two natural design questions which we study in this paper: when is it worth actively choosing the context, and what is the gain compared to passive sampling?

Besides a fully passive scenario for benchmarking, we study a fully active setting where the learner can choose both the subpopulation and the treatment as well as a budgeted intervention setting, in which only a fraction of the rounds can be active, capturing cases where targeted recruitment or controlled sampling is costly. Our analysis shows how the value of active subpopulation selection depends on the population distribution and when it can lead to meaningful improvements over passive sampling. Additionally, we consider the setting where subpopulation distribution is not a priori known and must be estimated. This setting introduces a trade-off between using passive samples to learn the subpopulation distribution $\mathbf p$ and using active samples to improve treatment recommendations. Somewhat surprisingly, we show that for sufficiently large horizons, this uncertainty does not result in worst simple regret rate compared to an active policy that knows the distribution.

Our work is related to several lines of work in literature. Heterogeneous treatment effect and subpopulation-aware decision making problems study settings where treatment effects vary across groups and where ignoring this heterogeneity can lead to suboptimal treatment rules \cite{russac2021b,curth2023adaptive}. Contextual bandits and policy learning methods also aim to learn context-dependent decisions, including recent work on simple regret in contextual bandits such as Proportional Response \cite{krishnamurthy2023proportional}, as well as adaptive experimental designs for policy learning \cite{kato2024adaptive,contextual-simple-regret-deshmukh2018simple,jorke2022simple}. These works typically consider the passive contextual setting, where contexts are revealed by the environment. Our focus is different: we quantify the improvement the learner can obtain by actively selecting which context (or subpopulation) to sample. The closest pure exploration formulation is multi-bandit best arm identification \cite{multibandit-bai-gabillon2011multi, multi-bandit-bai2-scarlett2019overlapping}, where the goal is to identify the best arm in each of several bandits. This is close to our fully active setting, but it studies best arm identification, where the goal is to minimize the error probability of identifying the best arm, rather than context-weighted simple regret, and does not study active and passive context sampling. We refer the reader to Appendix~\ref{apx:related_work} for a more detailed discussion of related work.

\textbf{Contributions.} Our main contributions are as follows:
\begin{itemize}[leftmargin=*, align=left]
    \item For known subpopulation distribution $\mathbf p$, we characterize the optimal context-weighted simple regret rate and show that it is attainable, up to constant factors, by a policy that selects which subpopulation to sample from independently of the observed history  proportionally to $p_j^{2/3}$, where $p_j$ is the $j$th element of the subpopulation distribution $\mathbf p$. 
   The gap of this active policy's regret compared to passive sampling can be of order $\Theta(k^{1/4})$, where $k$ is the number of contexts.
    
    \item We extend the analysis to a budgeted intervention setting, where the learner can actively choose the subpopulation only in  $\alpha$ fraction of the rounds. We characterize the optimal allocation in this setting, propose an algorithm that achieves the optimal simple regret rate, and identify the minimum intervention budget required to match the optimal fully active rate.
    
    \item We study the setting where the subpopulation distribution $\mathbf p$ is unknown to the learner. We show that neither fully passive nor fully active policies are optimal in general, and propose the EETC algorithm, which balances passive estimation with active allocation rounds. We prove that, for large horizons, EETC achieves the same rate as the optimal fully active policy that knows $\mathbf p$, up to constant factors.
\end{itemize}

\section{Problem Setup}\label{sec:setup}

In this section, we formally define the bandit with subpopulations problem considered in this paper and introduce the relevant notation.

\textbf{Notation.}
For any natural number $n$, let $[n] \coloneqq \{1,\dots,n\}$. Let $
\deln{n} \coloneqq \{ \mathbf x \in \mathbb{R}^n | \sum_i x_i = 1,\; x_i \geq 0 \},
$ denotes the probability simplex in $\mathbb{R}^n$ and $\Delta_+^{n-1} \subset \deln{n}$ the set of vectors with all positive entries. 
For a vector $\mathbf v \in \mathbb{R}^n$  we denote its $\ell_p$ norm by $\lpnorm{\mathbf v}{p}$.

The bandit with subpopulations problem models a sequential interaction between an agent and an environment. The environment consists of a treatment variable $A \in [n]$, a subpopulation (or context) variable $C \in [k]$, and a reward variable $Y \in [0,1]$. 
For each treatment--subpopulation pair $(i,j) \in [n] \times [k]$, the reward distribution is given by $\nu_{i,j}$ with mean $\mu_{i,j}$. Let $\bnu$ and $\bmu \in [0,1]^{n \times k}$ denote the matrix of reward distributions and reward means, respectively. We let $\mathbf p \in \delk_+$ denote the population distribution over subpopulations, that is,
$
\forall j \in [k]: \pr(C = j) = p_j>0
$
, and define
$
\pmin \coloneqq \min_{j \in [k]} p_j.
$
The agent interacts with the environment for $T$ rounds. At each round $t \in [T]$, the interaction proceeds as follows.
\begin{enumerate}[leftmargin=*, align=left]
    \item A subpopulation $C_t$ is either sampled according to $p$ and revealed to the agent, or selected by the agent based on the observed history $\hcal_t = (A_s, C_s, Y_s)_{s=1}^{t-1}$ and, when known, the population distribution vector $\mathbf p$.
    \item Based on the history $\hcal_t$ and the current subpopulation $C_t$, the agent selects a treatment $A_t$.
    \item The agent then observes a random reward $Y_t$, drawn independently from $\nu_{A_t,C_t}$.
\end{enumerate}

We refer to a round in which $C_t$ is sampled from $\mathbf p$ and the agent only chooses $A_t$ as \emph{passive} round. In contrast, when the agent is allowed to choose both $C_t$ and $A_t$, we call the round \emph{active}.

An instance of the problem is specified by a pair $(\bnu,\mathbf  p)$. Let $\mc{E}(n,k,\mathbf p)$ denote the class of all instances with $n$ treatments and $k$ subpopulations distributed according to  $\mathbf p$, where the reward distributions is supported on $[0,1]$. 
Although we focus on a finite discrete set of subpopulations, the same perspective also applies to continuous context spaces, as common in contextual bandits, by introducing a suitable discretization of the context space. Under such a discretization, we obtain the same  formulation and its corresponding regret bounds up to the discretization error.

After $T$ rounds, the agent outputs, for each subpopulation $j \in [k]$, a recommended best treatment $\ahatj$. The collection of sampling decisions made during the interaction together with the final recommendation rule forms a policy $\pi$. For a policy $\pi$ interacting with an instance $(\bnu,\mathbf p) \in \mc{E}(n,k,\mathbf p)$ with means matrix $\bmu$, we define the simple regret as the expected subpopulation-weighted gap between the optimal and recommended treatments:
\begin{align*}
    \srtpinup
    &\coloneqq \ebb_{j \sim \mathbf p}\brak{\mu_{\astarj,j} - \mu_{\ahatj,j}}= \sum_j p_j \ebb \brak{\Delta_j},
\end{align*}
where
$
\astarj \in \argmax_{i \in [n]} \mu_{i,j}
$
denotes an optimal treatment for subpopulation $j$, and $\Delta_j$ the recommendation gap for subpopulation $j$.
The objective is to minimize the worst-case simple regret (over the class $\mc{E}(n,k,\mathbf p)$) defined as:
\begin{align} \label{eq: worst-case regret}
    \srtpi \coloneqq \sup_{ \bnu: (\bnu,\mathbf p) \in \E(n,k,\mathbf p)} \srtpinup,
\end{align}
for a policy $\pi$, i.e., find $\arg\inf_{\pi} \srtpi.$

This criterion is worst-case with respect to the reward means, but remains parameterized by the subpopulation distribution $\mathbf p$. In other words, the  performance measure $\srtpi$ is  the supremum over reward configurations for a fixed population profile $\mathbf p$, rather than a worst-case  $\mathbf p$. This is because our objective is not to design algorithms for an adversarial subpopulation distribution, but develop algorithms whose performance \textit{adapts} to the instance-dependent vector $\mathbf p$. Thus, the problem is worst-case in terms of the mean rewards but instance-dependent for the subpopulation distribution.

Throughout the paper, we focus on the data-rich regime, in which the horizon %
$
T \gg n, k, 1/{\pmin}.
$

\subsection{Policy Classes}

We consider several classes of policies in the paper.
~A policy is called \emph{passive} if, at each round, the subpopulation is randomly drawn according to the distribution $\mathbf{p}$ and not selected by the agent. In contrast, a policy is called \emph{active} if the agent can pick the subpopulation at every round. We denote the classes of passive and active policies by $\Pi_{\mathrm{pas}}$ and $\Pi_{\mathrm{act}}$, respectively. 

We say that a policy $\pi$ is \emph{history-free} if its subpopulation selection rule does not depend on the observed history and may only depend on the distribution $\mathbf p$. 
Recall that each policy consists of two components: a subpopulation selection rule and an arm selection rule within each subpopulation. For history-free policies, only the first component is restricted to be independent of the history; the second component can clearly depend on the history. In all the algorithms considered in this paper, this second component is implemented through a standard bandit subroutine, as discussed in the next subsection.

We denote the class of all history-free policies by $\Piind{hf}$.
Note that  passive policies form a special case of history-free policies (i.e., $\Piind{pas} \subseteq \Piind{hf}$), since for such policies, the subpopulation is sampled according to $\mathbf p$ independently of the history. We denote the class of all policies, without any restriction on how the subpopulation selection rule may depend on observed rewards, by $\Pi$.

Finally, we distinguish between policies that know the distribution $\mathbf p$ and those that do not by indicating $\mathbf  p$ in parentheses. For example, $\Piind{hf}(\mathbf p)$ denotes the class of history-free policies that know $\mathbf p$, whereas $\Piind{hf}$ denotes the class of history-free policies that do not. We use the same convention for the other policy classes.

\subsection{Within-Subpopulation Bandit Subroutine}

Since we do not impose any structural relation across the reward means of different subpopulations, the learning problem within each subpopulation $j \in [k]$ can be viewed as an independent stochastic bandit problem with $n$ arms. If the number of samples allocated to subpopulation $j$ were fixed, then minimizing $\ebb[\Delta_j]$ would reduce to the classical simple regret problem in standard bandits, which is well-studied in the literature \cite{bubeck2009pure, zhao2023revisiting, liu2026efficient}. In our setting, however, these within-subpopulation horizons are induced by the learner's subpopulation selection rule. Thus, the main remaining question is how to allocate the total budget $T$ across subpopulations.

Throughout the paper, we use $\algsr$ to denote an anytime standard bandit subroutine for simple regret minimization. Each subpopulation maintains an independent copy of $\algsr$, which is used to select treatments whenever that subpopulation is sampled. This separates the standard within-subpopulation treatment selection problem from the higher level allocation problem.

We rely on the following standard minimax guarantee.

\begin{lemma}[\cite{bandit-book1-lattimore2020bandit}, Section 33]\label{lem:simple_regret}
For a stochastic bandit with $n$ arms and $T$ rounds, the minimax simple regret scales as $\Theta\big(\sqrt{n/T}\big)$. Moreover, this rate can be achieved by anytime algorithms, i.e., algorithms that do not require prior knowledge of $T$.
\end{lemma}

For concreteness, uniform sampling over the $n$ arms gives the optimal simple regret rate up to logarithmic factors. Alternatively, an anytime cumulative regret minimization algorithm such as MOSS \cite{moss-audibert2009minimax}, combined with a recommendation rule that outputs arms proportionally to their play counts, achieves the optimal $\sqrt{n/T}$ rate up to constants.

This reduction clarifies the allocation problem. Since the final objective weights the recommendation gap in subpopulation $j$ by $p_j$, one might expect the natural allocation to be proportional to $\mathbf p$, which is exactly the passive allocation. However, because simple regret decreases nonlinearly with the number of samples, this allocation is generally suboptimal. The next sections show that active subpopulation selection leads to a different allocation rule and can largely improve over passive sampling.

\section{Known Subpopulation Distribution}\label{sec:known_p}

In this section, we study the setting where the subpopulation distribution $\mathbf p$ is known. We first show that, under the worst-case simple regret criterion in \eqref{eq: worst-case regret}, it is sufficient to consider history-free policies: allowing the subpopulation selection rule to depend on observed history does not improve the minimax rate. We then compare passive and fully active policies and characterize the active--passive gap, showing that the optimal active allocation can improve over passive sampling by a factor as large as $\ocal(k^{1/4})$. Finally, we extend the analysis to the budgeted setting, where only an $\alpha$ fraction of rounds can be active, and characterize the optimal allocation under this constraint.

We begin with notation. For a policy $\pi$ interacting with an instance $(\bnu,\mathbf p)$, let $ \mathbf q(\pi,(\bnu,\mathbf p)) \in \delk $ denote the expected proportion of rounds allocated to the different subpopulations under $\pi$. More precisely, its $j$-th coordinate is the expected fraction of rounds $t$ in which $C_t = j$. As a consequence of our definition, for history-free policies, the allocation proportion is independent of $\bnu$, and we simply write $ \mathbf q(\pi,(\bnu,\mathbf p)) = \qpip. $

\subsection{Optimality of History-Free Policies}

We formalize the reduction to history-free policies. The key point is that, for the worst-case simple regret in \eqref{eq: worst-case regret}, the optimal rate is determined by the allocation of samples across subpopulations, and history-dependent changes to this allocation do not improve the minimax order.

\begin{restatable}[Simple Regret Lower Bound]{lemma}{rewardfreeOptimal}\label{lem: lower bound}
For any subpopulation distribution $\mathbf p \in \delkplus$ known to the agent and any policy $\pi \in \Pi(\mathbf p)$,
\begin{align*}
    \srtpi
    \in
    \Omega\Big(\sqrt{\frac{n}{T}}\, \lpnorm{\mathbf p}{2/3}\Big).
\end{align*}
\end{restatable}
\textbf{Proof sketch.} For any policy, we construct a family of $2^k$ hard instances indexed by the binary hypercube. In each instance, the rewards are binary with means equal to $1/2$ except for two candidate treatments in each subpopulation, one of which is optimal with a properly chosen gap; the identity of this optimal treatment varies across the family and depends on the policy. We then show that the average simple regret of the policy over this family is at least $\sqrt{\frac{n}{T}} \sum_{j \in [k]} \frac{p_j}{\sqrt{q_j}},$ where $q_j$ is the fraction of times the policy samples subpopulation $j$, averaged over this family of instances. Hence, the worst-case regret is at least of the same order. Optimizing this lower bound over allocations yields the choice $q_j \propto p_j^{2/3}$ and gives the minimax lower bound of order${\small \sqrt{\frac{n}{T}}\, \lpnorm{\mathbf p}{2/3}}$. The full proof alongside all the omitted proofs is deferred to Appendix \ref{apx: proofs}. \qed

We now show that, for history-free policies, the worst-case simple regret depends solely on the expected allocation proportions across subpopulations, namely on the vector $\qpip$. The following characterization is tight up to constant factors.

\begin{restatable}[Simple Regret of History-Free Policies]{lemma}{rewardfreeSimpleRegret}\label{lem:history-free simple regret}
For any history-free policy $\pi \in \Piind{hf}(\mathbf p)$ along with the subroutine $\algsr$ and expected subpopulation proportions $\qpip$, if $\forall j \in [k] :  T \qjpip > 24 \, \ln(2kT)$, then
\begin{align*}
    \srtpi \in \Theta\Big(\sqrt{\frac{n}{T}} \sum_{j \in [k]} \frac{p_j}{\sqrt{\qjpip}}\Big).
\end{align*}
\end{restatable}
This shows that, for history-free policies, the performance depends solely on the quantity 
$\sum_{j \in [k]} p_j/\sqrt{\qjpip}$. Therefore, for any $\mathbf q \in \delkplus$, we define
\begin{align}\label{eq:spq}
    S_{\mathbf p}(\mathbf q) \coloneqq \sum_{j \in [k]} \frac{p_j}{\sqrt{q_j}}.
\end{align}
The following lemma identifies the optimal active allocation that minimizes this quantity.

\begin{restatable}[Optimal Allocation for Active Policies]{lemma}{optimalQSp}\label{lem:optimal_q_sp}
For any $\mathbf p \in \delkplus$, the unique minimizer of $S_{\mathbf p}(\mathbf q)$ over $\mathbf q \in \delkplus$ is
$$
q^*(\mathbf p)_j
=
\frac{p_j^{2/3}}{\sum_{\ell \in [k]} p_\ell^{2/3}},
\quad j \in [k].
$$
Moreover,
$
S_{\mathbf p}(\mathbf q^*(\mathbf p)) = \lpnorm{\mathbf p}{2/3}.
$
\end{restatable}

Combining Lemmas~\ref{lem: lower bound}, \ref{lem:history-free simple regret}, and \ref{lem:optimal_q_sp} gives the following optimality result, which states that history-free policies achieve the same regret as general policies that can select the subpopulation allocation based on the history.

\begin{proposition}[Optimality of History-Free Policies]\label{cor:reduction}
For every $\mathbf p \in \delkplus$, \begin{align*} \inf_{\pi \in \Piind{hf}(\mathbf p)} \srtpi \in \Theta\Big(\inf_{\pi \in \Pi(\mathbf p)} \srtpi\Big). \end{align*} That is, there is no optimality penalty for adhering to history-free policies.\end{proposition}


\subsection{Gap Between Active and Passive Policies} 

We now compare the best passive and active policies within the class of history-free policies that know $\mathbf p$. Under any passive policy, the subpopulation observed at each round is sampled from $\mathbf p$, and hence $\qpip=\mathbf p$. Lemma~\ref{lem:history-free simple regret} gives
\begin{align}\label{eq:passive regret bound}
    \inf_{\pi \in \Piind{pas}(\mathbf p)} \srtpi
    \in
    \Theta\Big(
        \sqrt{\frac{n}{T}}
        \sum_{j \in [k]} \sqrt{p_j}
    \Big)
    =
    \Theta\Big(
        \sqrt{\frac{n}{T}}\, \lpnorm{\mathbf p}{1/2}^{1/2}
    \Big).
\end{align}

\begin{figure}[t]
    \centering
    \begin{subfigure}[b]{0.47\textwidth}
        \centering
        \includegraphics[width=\textwidth]{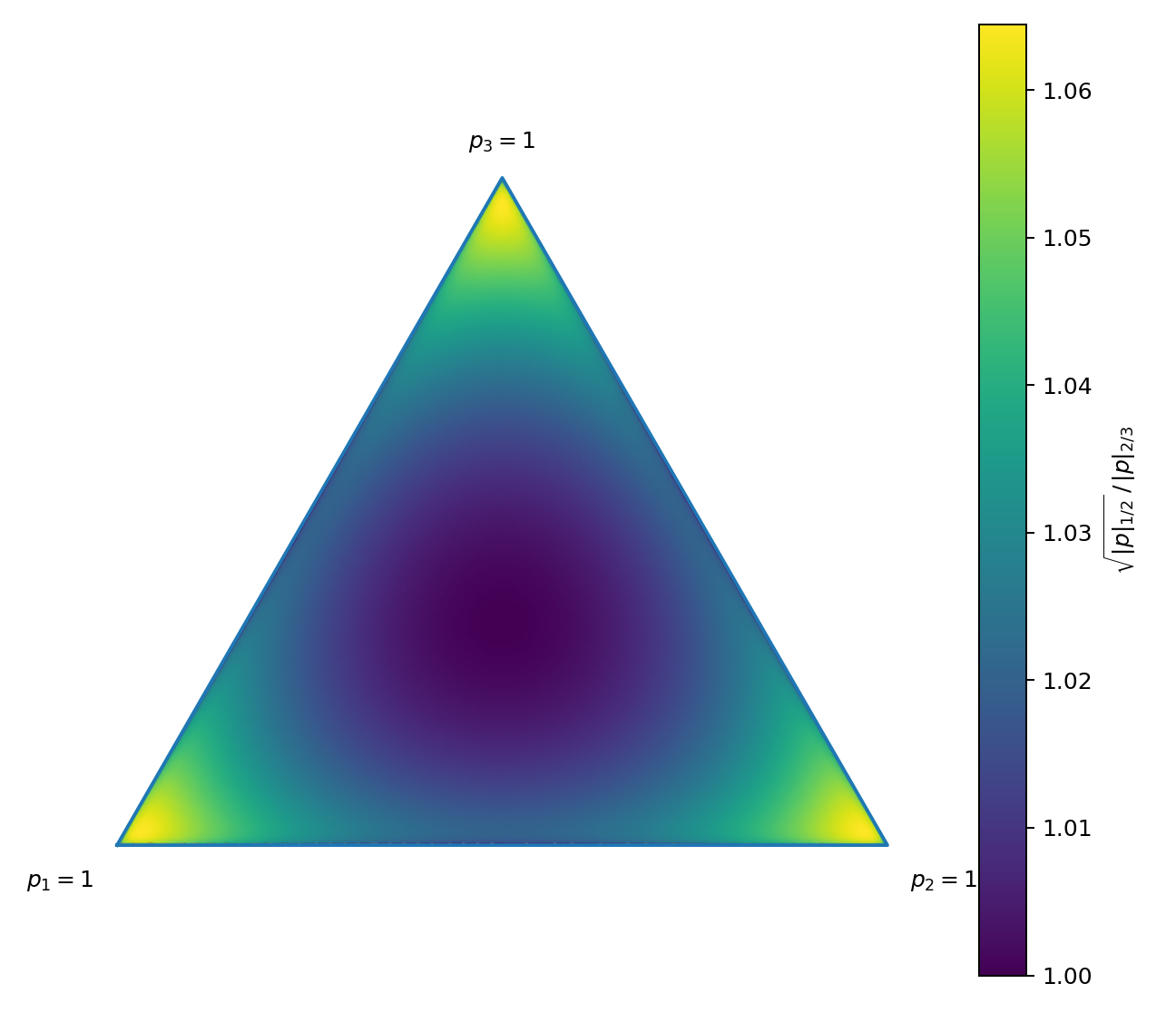}
        \caption{Heatmap of $R(\mathbf p)$ on the $3$-simplex.}
    \end{subfigure}
    \hspace{1cm}
    \begin{subfigure}[b]{0.37\textwidth}
        \centering
        \includegraphics[width=\textwidth]{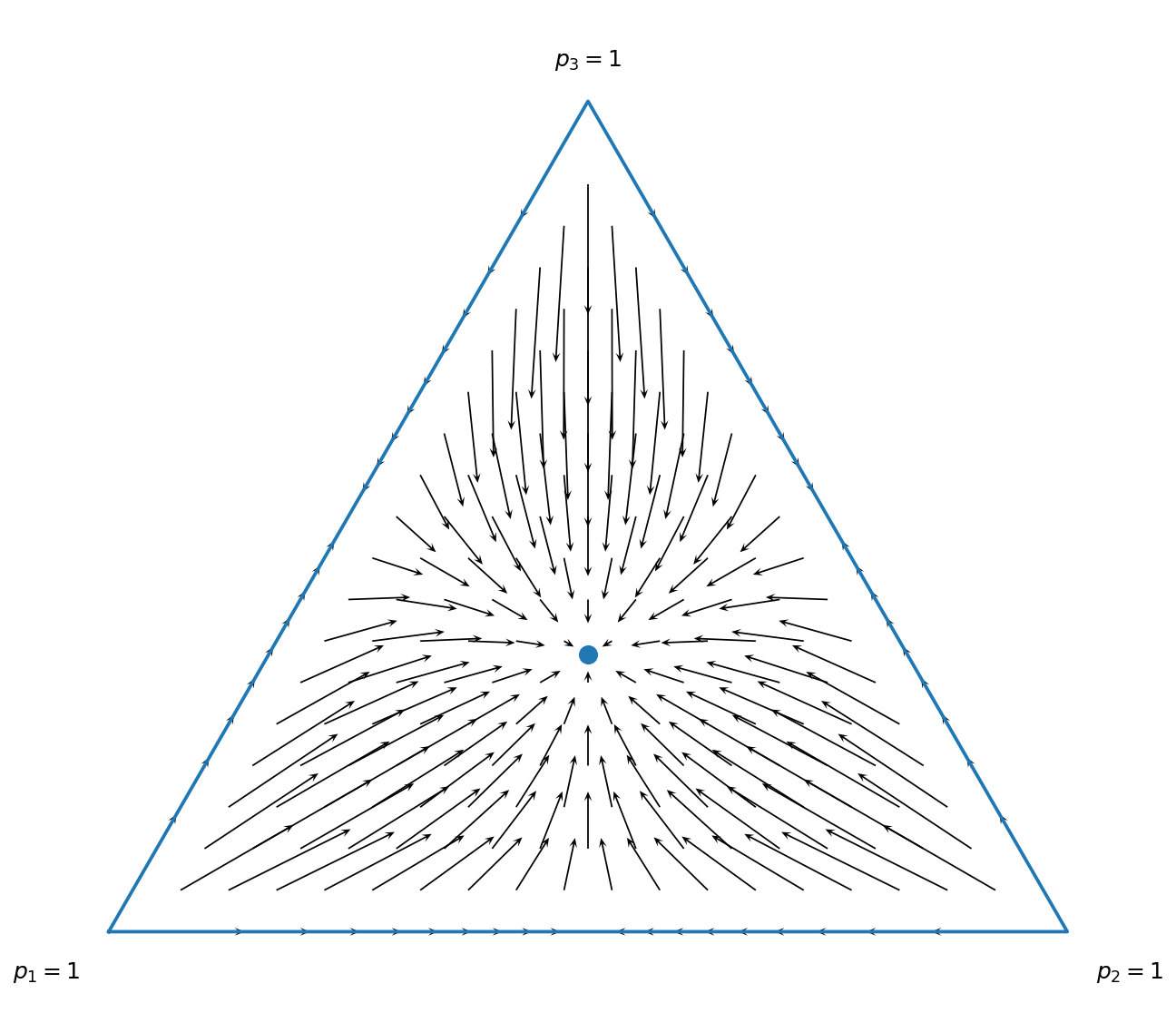}
        \caption{Flow from $\mathbf p$ to the optimal active allocation $\mathbf q^*(\mathbf p)$.}
    \end{subfigure}
    \caption{Illustration of the active--passive gap on the probability simplex for $k=3$. Left: heatmap of $R(\mathbf p)$ over the simplex, showing that the gap is negligible near the uniform distribution (center) and increases for non-uniform population profiles (corners). 
    Right: for each point $\mathbf p$, the arrow points toward the optimal active allocation $\mathbf q^*(\mathbf p)$, illustrating that the optimal active policy shifts the sampling proportions toward a more uniform (center), and hence more exploratory, allocation.}
    \label{fig:active_passive_simplex}
\end{figure}

On the other hand, Lemmas~\ref{lem:history-free simple regret} and~\ref{lem:optimal_q_sp} imply that the best active history-free policy satisfies
\begin{align}\label{eq:active regret bound}
    \inf_{\pi \in \Piind{act}(\mathbf p) \cap \Piind{hf}(\mathbf p)} \srtpi
    \in
    \Theta\Big(
        \sqrt{\frac{n}{T}}\, \lpnorm{\mathbf p}{2/3}
    \Big).
\end{align}
This rate is achieved by sampling subpopulation $j$ proportionally to $p_j^{2/3}$ and running $\algsr$ within each subpopulation. Compared to the passive allocation $\mathbf p$, the optimal active allocation is closer to uniform: it downweights large coordinates of $\mathbf p$ and upweights small ones, leading to a more exploratory allocation. See Figure~\ref{fig:active_passive_simplex} for an illustration.

The multiplicative active--passive gap is therefore
$
R(\mathbf p) \coloneqq {\sqrt{\lpnorm{\mathbf p}{1/2}}}/{\lpnorm{\mathbf p}{2/3}} 
$
and \eqref{eq:passive regret bound}--\eqref{eq:active regret bound} imply
\begin{align*}
    \frac{\inf_{\pi \in \Piind{pas}(\mathbf p)} \srtpi}
    {\inf_{\pi \in \Piind{act}(\mathbf p)\cap\Piind{hf}(\mathbf p)} \srtpi}
    \in
    \Theta\big(R(\mathbf p)\big).
\end{align*}

\begin{restatable}[Active--Passive Gap]{lemma}{activePassiveGap}\label{lem:active_passive_gap}
For every $\mathbf p \in \delkplus$,
$
R(\mathbf p) \geq 1.
$
Moreover,
$
\sup_{\mathbf p \in \delkplus} R(\mathbf p) \in \Theta\big(k^{\frac14}\big),
$
and this order is tight.
\end{restatable}

The inequality $R(\mathbf p) \geq 1$ implies that active policies are always at least as good as passive policies. When $\mathbf p$ is uniform, $R(\mathbf p)=1$, so the two rates coincide; more generally, the gain is limited near uniform distributions. The largest gap arises for highly non-uniform $\mathbf p$,  in particular when one subpopulation has a large mass and the remaining coordinates are all equal and of order $k ^ {-3/2}$.

\subsection{Budgeted Interventions}\label{sec:budget}

In many practical settings, the agent has a limited budget of active rounds, in which it can choose both the subpopulation and the treatment. We focus on such budgeted setting in this section. More precisely, we consider policies that are allowed to actively select the subpopulation in at most $\alpha T$ of the total $T$ rounds, where $\alpha \in [0,1]$. We assume the distribution $\mathbf p$ is known to the agent, and denote this class of policies by $\Piind{\alpha}(\mathbf p)$. Note that $\Piind{\alpha=0}(\mathbf p)=\Piind{pas}(\mathbf p)$ and $\Piind{\alpha=1}(\mathbf p)=\Piind{act}(\mathbf p)$, so this model interpolates between the passive and fully active settings.

Our algorithm design relies on the result of Lemma~\ref{lem:history-free simple regret}, which states that the regret of a history-free policy $\pi$ depends on the subpopulation proportions $\qpip$ only through the quantity
$S_{\mathbf p}(\mathbf q)$ defined in \eqref{eq:spq}.
Thus, the goal is to choose the allocation vector $\mathbf q$ to minimize this term. Given that the subpopulation cannot be selected by the agent during the $(1-\alpha)T$ rounds but must instead be sampled according to $\mathbf p$, any policy in $\Piind{\alpha}(\mathbf p)$ must satisfy 
$\qjpip \geq (1-\alpha)p_j$, $\forall j \in [k]$.
Therefore, the best achievable allocation is given by the solution of the following constrained optimization problem,
\begin{align} \label{eq:budget optimization}
    \min_{\mathbf q \in \delkplus} S_{\mathbf p}(\mathbf q)=\min_{\mathbf q \in \delkplus} \sum_{j \in [k]} \frac{p_j}{\sqrt{q_j}},\qquad  
    \text{s.t.} \quad q_j \geq (1-\alpha)p_j, \quad \forall j \in [k].
\end{align}
This optimization problem is convex and its solution admits a simple threshold structure. Let $\qstalphap$ and $\vstalphap$ denote the optimizer and the optimal value of the program, respectively. In particular, there exists a constant $c^* \geq 0$ such that
$
\qstalphap_j = \max\big((1-\alpha)p_j,\; c^* p_j^{2/3}\big), j \in [k].
$
Thus, the subpopulations are partitioned into two groups. For those subpopulations whose coordinates satisfy $\qstalphap_j > (1-\alpha)p_j$, the optimal allocation remains proportional to $p_j^{2/3}$, whereas for subpopulations satisfying $\qstalphap_j = (1-\alpha)p_j$, the passive rounds already provide a sufficient number of samples, so there is no need to actively allocate additional pulls to them. Algorithm~\ref{algo:budget}, which is based on this characterization, summarizes the steps of the budgeted intervention setting.

\begin{algorithm}[t]
    \caption{$\alpha$-Active Simple Regret Algorithm}
    \label{algo:budget}
    \begin{algorithmic}[1]
        \STATE \textbf{Input:} Standard bandit subroutine $\algsr$, distribution $\mathbf p$, budget $\alpha$, and horizon $T$.
        \STATE Initialize one instance $\algsr^j$ of $\algsr$ for each subpopulation $j \in [k]$.
        \STATE Define $F(c) = \sum_{j \in [k]} \max \big((1-\alpha)p_j,\; c p_j^{2/3}\big)$.
        \FOR{$t = 1,2,\dots,\lfloor(1-\alpha)T\rfloor$}
            \STATE Sample $C_t \sim \mathbf p$, and choose the treatment $A_t$ recommended by $\algsr^{C_t}$.
        \ENDFOR
        \STATE Find the unique value $c^*$ such that $F(c^*) = 1$.
        \STATE For every $j \in [k]$, set
        $
        q_j^* = \max \big((1-\alpha)p_j,\; c^* p_j^{2/3}\big),$
        and 
        $
        r_j = \frac{q_j^* - (1-\alpha)p_j}{\alpha}.
        $
        \FOR{$t = \lfloor(1-\alpha)T\rfloor + 1,\dots,T$}
            \STATE Sample subpopulation $C_t$ according to $\mb{r}$, and choose the treatment $A_t$ recommended by $\algsr^{C_t}$.
        \ENDFOR
        \STATE \textbf{Output:} For each subpopulation $j$, output the recommended arm $\ahatj$ returned by $\algsr^j$.
    \end{algorithmic}
\end{algorithm}

\begin{restatable}[Simple Regret of Algorithm~\ref{algo:budget}]{theorem}{budget}\label{thm:budget}
For every $\mathbf p \in \delkplus$, Algorithm~\ref{algo:budget} is a history-free policy in $\Piind{\alpha}(\mathbf p)$ with expected subpopulation proportions equal to $\qstalphap$. Consequently, when $T$ is large, its worst-case simple regret satisfies the optimal rate
\begin{equation*}
    {\text{SR}(\texttt{Alg.1}, T, \mathbf p)}
    \in \ocal\big(\sqrt{\frac{n}{T}} \, \, \vstalphap\big).
\end{equation*}
\end{restatable} 
Another question that arises naturally in this setting is that given a subpopulation distribution $\mathbf{p}$, what is the minimum intervention budget $\alpha$ required 
to match the fully active setting? Equivalently, how many active rounds are sufficient to attain the optimal fully active rate. We denote this threshold by $\alphmin$. The next lemma gives a closed-form expression for $\alphmin$ in terms of $\mathbf p$.

\begin{restatable}[Threshold Budget $\alphmin$]{lemma}{alphamin}\label{lem:alphamin}
For a given $\mathbf p \in \delkplus$, let
$
\alphmin \coloneqq 1 - {p_{\max}^{-1/3}}/{\sum_{j \in [k]} p_j^{2/3}},
$
where $p_{\max} = \max_j p_j$.
Then, for every $\alpha \in [\alphmin,1]$, we have
$
\vstalphap = \lpnorm{\mathbf p}{2/3}.
$
Consequently, for every such $\alpha$, the simple regret bound of Algorithm~\ref{algo:budget} matches the optimal rate of fully active policies, which is also the optimal rate over all policies.
\end{restatable}

\section{Unknown Subpopulation Distribution}\label{sec:unknown_p}

In this section, we study the setting in which the distribution of subpopulations, $\mb{p}$, is unknown to the learner, and investigate how this changes the policy design and the achievable regret rates. We show that neither passive nor fully active policies are optimal in general. Instead, the optimal strategy must combine passive and active rounds. We then propose an algorithm based on this principle and show that, for a sufficiently large horizon $T$, it achieves the same simple regret bound as a fully active policy that knows $\mb{p}$.

The simple regret criterion in \eqref{eq: worst-case regret} depends explicitly on $\mb{p}$. Yet a fully active policy agnostic of $\mb{p}$ that chooses a subpopulation at every round,  receives no information about $\mb{p}$. Consequently, if the policy undersamples some subpopulation, the worst-case instance may assign a large mass to that subpopulation, resulting in a large regret. Therefore, without knowledge of $\mb{p}$, the minimax fully active policy must allocate samples uniformly across subpopulations, i.e., $\qpip = \frac{1}{k}\mb{1}$, where $\mb{1}$ denotes the all-ones vector. By Lemma~\ref{lem:history-free simple regret}, the resulting simple regret is
\begin{equation*}
\Theta\Big(\sqrt{\frac{n}{T}} \sum_{j \in [k]} \frac{p_j}{\sqrt{1/k}}\Big)
= \Theta\Big(\sqrt{\frac{nk}{T}}\Big),
\end{equation*}
which is worse than the passive regret rate
$
\Theta\big(\sqrt{\frac{n}{T}\lpnorm{\mb{p}}{\frac12}} \big)
$ derived in \eqref{eq:passive regret bound}, as
$
\sum_{j \in [k]}\! \sqrt{p_j}\! \leq\! \sqrt{k},
$
for every $\mb{p}\! \in\! \delkplus$.

In the next section, we present an algorithm that outperforms the passive policy. It does so by balancing the exploration–exploitation trade-off in learning the distribution $\mathbf{p}$. Passive rounds provide information about $\mathbf{p}$ and allow the learner to estimate it, while active rounds are used to reduce the simple regret once a sufficiently accurate estimate is available.

\subsection{Explore-Explore-Then-Commit (EETC)}
For each round $t \in [T]$ and subpopulation $j \in [k]$, let
$
\ntj = \sum_{s=1}^t \mb{1}\{C_s = j\}
$
denote the number of times subpopulation $j$ has been observed up to round $t$, including round $t$ itself, and let $\pht$ denote the empirical estimate of $\mb{p}$, namely $\pht_j = \ntj/t$. 

The algorithm proceeds in three phases: (I) The algorithm acts passively in order to estimate $\mb{p}$. The key point is that after $\mathcal{O}(\log(T))$ rounds, the estimate becomes accurate enough so that the allocation $\mb{q}^*(\pht)$ computed from $\pht$ yields a value $S_{\mb{p}}(\mb{q}^*(\pht))$ within a constant factor of the optimal value $\lpnorm{\mb{p}}{2/3}$. Motivated by this, we define the first stopping time
\begin{align}
    \label{eq:tau_1 def}
    \tau_1 \coloneqq \inf \Big\{ 1 \leq t \leq T \mid \min_{j \in [k]} \ntj \geq \log(kT^2) \Big\}.
\end{align}
We show that, with high probability, $\tau_1$ is logarithmic in $T$, and that at time $\tau_1$ the estimate $\phtau{1}$ is sufficiently accurate for the subsequent allocation steps.

(II) This phase commences at time $\tau_1$ and the algorithm still acts passively. While at this point, the estimate is already accurate enough to guide the active allocation, the algorithm can still benefit from additional passive samples as long as doing so does not leave too few active rounds to implement the optimal active allocation which is suggested by the current estimate. Lemma~\ref{lem:alphamin} gives exactly this threshold: for an estimate $\pht$, the fully active rate can still be matched if the remaining active fraction is at least $\alphmininp{\pht}$, or equivalently if the passive fraction is at most $1-\alphmininp{\pht}$.
Accordingly, we define
\begin{align}
    \label{eq:tau_2 def}
    \tau_2 \coloneqq \inf \big\{ t \geq \tau_1 + 1 \mid  
    {t}/{T} > 1 - \alphmininp{\pht} \big\}.
\end{align}
Thus, $\tau_2$ is the first round at which the current estimate no longer allows us to keep playing passively and still match the optimal active rate. The algorithm therefore stops the passive phase at round $\tau_2$ and sets
$
\phfreeze \coloneqq \hat{\mb{p}}(\tau_2-1),
$ where $\hat{\mb{p}}(\tau_2-1)$ is the estimate of the subpopulation distribution vector at the previous round, i.e., $\tau_2 -1$.
For large $T$, we show that $\tau_2 \leq T$ with high probability.

\begin{algorithm}[H]
    \caption{Explore-Explore-Then-Commit (EETC)}
    \label{algo:unknown}
    \begin{algorithmic}[1]
        \STATE \textbf{Input:} standard bandit subroutine $\algsr$ and horizon $T$.
        \STATE Initialize one instance $\algsr^j$ of $\algsr$ for each subpopulation $j \in [k]$.
        \FOR{$t = 1,\dots,\tau_2$}
            \STATE Sample $C_t \sim \mb{p}$ and choose the treatment $A_t$ recommended by $\algsr^{C_t}$.
        \ENDFOR
        \STATE Compute $\mb{q}$ and $\mb{r}$ from \eqref{eq:q and r for unknown} using $\phfreeze$.
        \FOR{$t = \tau_2 + 1,\dots,T$}
            \STATE Sample subpopulation $C_t$ according to $\mb{r}$ and choose the treatment $A_t$ recommended by $\algsr^{C_t}$.
        \ENDFOR
        \STATE \textbf{Output:} For each subpopulation $j$, output the recommended arm $\ahatj$ returned by $\algsr^j$.
    \end{algorithmic}
\end{algorithm}

(III) The algorithm switches to active play using $\phfreeze$. The remaining rounds are then allocated exactly as in the budgeted-intervention algorithm of Section~\ref{sec:budget}, with active fraction $1 - \frac{\tau_2}{T}$ and estimated distribution $\phfreeze$. 
Specifically, for each $c > 0$, define
$$
F(c) := \sum_{j \in [k]} \max \{ \frac{\tau_2}{T}\hat{p}^{\mathrm{fr}}_j,\;   (\hat{p}^{\mathrm{fr}}_j)^{2/3}c \},
$$
let $c^* $ be the unique value such that $F(c^*) = 1$, and define the vectors $\mb{q}$ and $\mb{r}$ by
\begin{align}
    \label{eq:q and r for unknown}
    q_j := \max \Big\{ \frac{\tau_2}{T}\hat{p}^{\mathrm{fr}}_j,\;  (\hat{p}^{\mathrm{fr}}_j)^{2/3}c^* \Big\},
    \qquad
    r_j = \frac{q_j - \frac{\tau_2}{T}\hat{p}^{\mathrm{fr}}_j}{1 - {\tau_2}/{T}}.
\end{align}
The algorithm then uses the remaining rounds actively, selecting subpopulations according to $\mb{r}$. Algorithm~\ref{algo:unknown} presents the pseudocode, and the next theorem states its regret guarantee.

\begin{restatable}[Simple Regret of EETC]{theorem}{unknownP}\label{thm:unknown_p}
For sufficiently large horizon $T$, the worst-case simple regret of EETC is
\begin{equation*}
   {\text{SR}(\texttt{EETC}, T, \mathbf p)} \in O \big(\sqrt{\frac{n}{T}} \, \lpnorm{\mb{p}}{2/3}\big).
\end{equation*}
\end{restatable}
Consequently, EETC up to a constant achieves the same simple regret rate as the optimal fully active policy that knows $\mb{p}$, and hence also matches the known-$p$ lower bound up to constant factors.

\section{Experiments}\label{sec:experiments}

This section presents the experimental results on real-world data. Additional dataset details and further experiments are deferred to Appendix \ref{apx:exp}.

\textbf{MovieLens Experiment.}
We evaluate the algorithms on a MovieLens 1M instance \cite{harper2015movielens}. Treatments correspond to the five most frequently rated movies, and subpopulations correspond to the $14$ demographic groups obtained from the Cartesian product of gender and age group. Thus, in this experiment, $n=5$ and $k=14$, yielding $nk=70$ different treatment-subpopulation reward distributions. When an algorithm selects treatment $a$ for subpopulation $j$, the environment samples uniformly with replacement from the historical ratings matching $(a,j)$ and returns the sampled rating as the reward. The population weights $p_j$ and the mean rewards $\mu_{a,j}$ are computed from the empirical frequencies and cell averages of the filtered dataset.  For more details see Appendix~\ref{apx:exp}.

Within each subpopulation, we instantiate the standard bandit subroutine in two ways: \emph{Uniform}, which samples the available treatments uniformly, and \emph{UCB}, which uses the standard upper-confidence-bound rule. Both are known to achieve the optimal simple regret rate up to logarithmic factors. We compare three policy classes: fully passive, fully active with known $\mb{p}$ using the optimal allocation proportional to $p_j^{2/3}$, and EETC. For each policy class, we report results with both subroutines, yielding six methods in total.

\textbf{Results.}
Table~\ref{tab:movielens_main} reports the simple regret after horizons $T \in \{2500,5000,10000,15000\}$, averaged over $50$ independent runs, together with confidence intervals. Across both subroutines, the passive policies perform worst, while EETC and the known-$p$ active policies achieve clearly smaller regret. Moreover, as the horizon increases, the performance of EETC becomes very close to that of the known-$p$ active benchmark. These results consistent with the theory showcase that estimating $p$ from passive observations and then switching to an active allocation is more efficient than remaining fully passive, and that EETC can closely approach the performance of the optimal known-$p$ active policy.


\begin{table}[t]
    \centering
    \small
    \caption{Simple regret on the MovieLens 1M instance \cite{harper2015movielens}. Each entry reports mean $\pm$ 95\% confidence interval over $50$ independent runs. All values are multiplied by $100$ for readability.}
    \label{tab:movielens_main}
    \begin{tabular}{lcccc}
        \toprule
        Method & $T=2500$ & $T=5000$ & $T=10000$ & $T=15000$ \\
        \midrule
        Passive + Uniform & 3.76 $\pm$ 0.53 & 2.05 $\pm$ 0.32 & 1.10 $\pm$ 0.23 & 0.75 $\pm$ 0.13 \\
        EETC + Uniform & 3.52 $\pm$ 0.52 & 1.85 $\pm$ 0.32 & 0.81 $\pm$ 0.25 & 0.59 $\pm$ 0.13 \\
        Known-$p$ Active + Uniform & \textbf{2.99 $\pm$ 0.51} & \textbf{1.78 $\pm$ 0.37} & \textbf{0.74 $\pm$ 0.17} & \textbf{0.56 $\pm$ 0.14} \\
        \midrule
        Passive + UCB & 2.53 $\pm$ 0.37 & 1.29 $\pm$ 0.23 & 0.71 $\pm$ 0.13 & 0.47 $\pm$ 0.11 \\
        EETC + UCB & 2.57 $\pm$ 0.37 & 0.97 $\pm$ 0.17 & \textbf{0.41 $\pm$ 0.11} & 0.33 $\pm$ 0.10 \\
        Known-$p$ Active + UCB & \textbf{2.28 $\pm$ 0.46} & \textbf{0.88 $\pm$ 0.21} & 0.49 $\pm$ 0.12 & \textbf{0.31 $\pm$ 0.07} \\
        \bottomrule
    \end{tabular}
\end{table}

\section{Conclusion}

We studied the bandit with subpopulations problem with a context-weighted simple regret objective and quantified the benefit of actively choosing which subpopulation to sample. We showed that active subpopulation selection always improves performance over passive sampling, characterized the extent of this improvement, and proved the optimality of a simple history-free allocation rule up to constant factors. Moreover, we extended the analysis to settings with a limited intervention budget and unknown subpopulation distribution. In the latter case, we proposed the EETC algorithm that balances passive estimation with active allocation rounds and for long horizons matches the performance of the optimal fully active policy which knows the population distribution.
An interesting direction for future work is to develop gap-dependent guarantees for the same problem. In such a setting, optimal algorithms may allocate fewer samples to subpopulations where the best treatment is easy to identify, leading to allocation rules that differ from the worst-case ones studied here. Another possible extension is to consider models where the context is observed only after the treatment is chosen, inspired by related contextual and causal bandit settings \cite{russac2021b, lattimore2016causal, shahverdikondori2025optimal}.

\bibliographystyle{alpha}
\bibliography{biblio}

@book{bandit-book1-lattimore2020bandit,
  title={Bandit algorithms},
  author={Lattimore, Tor and Szepesv{\'a}ri, Csaba},
  year={2020},
  publisher={Cambridge University Press}
}

@article{bandit-book2-bubeck2012regret,
  title={Regret analysis of stochastic and nonstochastic multi-armed bandit problems},
  author={Bubeck, S{\'e}bastien and Cesa-Bianchi, Nicolo and others},
  journal={Foundations and Trends{\textregistered} in Machine Learning},
  volume={5},
  number={1},
  pages={1--122},
  year={2012},
  publisher={Now Publishers, Inc.}
}

@article{lattimore2016causal,
  title={Causal bandits: Learning good interventions via causal inference},
  author={Lattimore, Finnian and Lattimore, Tor and Reid, Mark D},
  journal={Advances in neural information processing systems},
  volume={29},
  year={2016}
}

@article{pomis-lee2018structural,
  title={Structural causal bandits: Where to intervene?},
  author={Lee, Sanghack and Bareinboim, Elias},
  journal={Advances in neural information processing systems},
  volume={31},
  year={2018}
}

@inproceedings{CUCB-lu2020regret,
  title={Regret analysis of bandit problems with causal background knowledge},
  author={Lu, Yangyi and Meisami, Amirhossein and Tewari, Ambuj and Yan, William},
  booktitle={Conference on Uncertainty in Artificial Intelligence},
  pages={141--150},
  year={2020},
  organization={PMLR}
}

@inproceedings{budget-nair2021budgeted,
  title={Budgeted and non-budgeted causal bandits},
  author={Nair, Vineet and Patil, Vishakha and Sinha, Gaurav},
  booktitle={International Conference on Artificial Intelligence and Statistics},
  pages={2017--2025},
  year={2021},
  organization={PMLR}
}

@inproceedings{jamshidi2024confounded,
  title={Confounded budgeted causal bandits},
  author={Jamshidi, Fateme and Etesami, Jalal and Kiyavash, Negar},
  booktitle={Causal Learning and Reasoning},
  pages={423--461},
  year={2024},
  organization={PMLR}
}

@article{shahverdikondori2025optimal,
  title={Optimal Best Arm Identification with Post-Action Context},
    author={Mohammad Shahverdikondori and Amir Mohammad Abouei and Alireza Rezaeimoghadam and Negar Kiyavash},
    journal={arXiv preprint arXiv:2502.03061},
    year={2025},
    eprint={2502.03061},
    archivePrefix={arXiv},
    primaryClass={cs.LG},
}

@article{perchet2013multi,
  title={THE MULTI-ARMED BANDIT PROBLEM WITH COVARIATES},
  author={Perchet, Vianney and Rigollet, Philippe},
  journal={THE ANNALS of STATISTICS},
  pages={693--721},
  year={2013},
  publisher={JSTOR}
}

@article{russac2021b,
  title={A/b/n testing with control in the presence of subpopulations},
  author={Russac, Yoan and Katsimerou, Christina and Bohle, Dennis and Capp{\'e}, Olivier and Garivier, Aur{\'e}lien and Koolen, Wouter M},
  journal={Advances in Neural Information Processing Systems},
  volume={34},
  pages={25100--25110},
  year={2021}
}

@inproceedings{wu2023best,
  title={Best Arm Identification with Fairness Constraints on Subpopulations},
  author={Wu, Yuhang and Zheng, Zeyu and Zhu, Tingyu},
  booktitle={2023 Winter Simulation Conference (WSC)},
  pages={540--551},
  year={2023},
  organization={IEEE}
}

@inproceedings{curth2023adaptive,
  title={Adaptive identification of populations with treatment benefit in clinical trials: machine learning challenges and solutions},
  author={Curth, Alicia and H{\"u}y{\"u}k, Alihan and Van Der Schaar, Mihaela},
  booktitle={International Conference on Machine Learning},
  pages={6603--6622},
  year={2023},
  organization={PMLR}
}

@article{bickel2009discriminative,
  title={Discriminative learning under covariate shift.},
  author={Bickel, Steffen and Br{\"u}ckner, Michael and Scheffer, Tobias},
  journal={Journal of Machine Learning Research},
  volume={10},
  number={9},
  year={2009}
}

@article{kato2022best,
  title={Best arm identification with contextual information under a small gap},
  author={Kato, Masahiro and Imaizumi, Masaaki and Ishihara, Takuya and Kitagawa, Toru},
  journal={arXiv preprint arXiv:2209.07330},
  year={2022}
}

@article{foster2019variational,
  title={Variational Bayesian optimal experimental design},
  author={Foster, Adam and Jankowiak, Martin and Bingham, Elias and Horsfall, Paul and Teh, Yee Whye and Rainforth, Thomas and Goodman, Noah},
  journal={Advances in neural information processing systems},
  volume={32},
  year={2019}
}

@inproceedings{jorke2022simple,
  title={Simple regret minimization for contextual bandits using bayesian optimal experimental design},
  author={J{\"o}rke, Matthew and Lee, Jonathan and Brunskill, Emma},
  booktitle={ICML2022 Workshop on Adaptive Experimental Design and Active Learning in the Real World},
  year={2022}
}

@article{kato2024adaptive,
  title={Adaptive experimental design for policy learning},
  author={Kato, Masahiro and Okumura, Kyohei and Ishihara, Takuya and Kitagawa, Toru},
  journal={arXiv preprint arXiv:2401.03756},
  year={2024}
}

@inproceedings{russo2024fair,
  title={Fair best arm identification with fixed confidence},
  author={Russo, Alessio and Vannella, Filippo},
  booktitle={2024 IEEE 63rd Conference on Decision and Control (CDC)},
  pages={1173--1180},
  year={2024},
  organization={IEEE}
}

@book{hardy1952inequalities,
  title={Inequalities},
  author={Hardy, G. H. and Littlewood, J. E. and P{\'o}lya, G.},
  edition={2},
  year={1952},
  publisher={Cambridge University Press}
}

@book{dubhashi2009concentration,
  title={Concentration of measure for the analysis of randomized algorithms},
  author={Dubhashi, Devdatt P and Panconesi, Alessandro},
  year={2009},
  publisher={Cambridge University Press}
}

@article{krishnamurthy2023proportional,
  title={Proportional response: Contextual bandits for simple and cumulative regret minimization},
  author={Krishnamurthy, Sanath Kumar and Zhan, Ruohan and Athey, Susan and Brunskill, Emma},
  journal={Advances in Neural Information Processing Systems},
  volume={36},
  pages={30255--30266},
  year={2023}
}

@article{contextual-simple-regret-deshmukh2018simple,
  title={Simple regret minimization for contextual bandits},
  author={Deshmukh, Aniket Anand and Sharma, Srinagesh and Cutler, James W and Moldwin, Mark and Scott, Clayton},
  journal={arXiv preprint arXiv:1810.07371},
  year={2018}
}

@article{multibandit-bai-gabillon2011multi,
  title={Multi-bandit best arm identification},
  author={Gabillon, Victor and Ghavamzadeh, Mohammad and Lazaric, Alessandro and Bubeck, S{\'e}bastien},
  journal={Advances in Neural Information Processing Systems},
  volume={24},
  year={2011}
}

@inproceedings{multi-bandit-bai2-scarlett2019overlapping,
  title={Overlapping multi-bandit best arm identification},
  author={Scarlett, Jonathan and Bogunovic, Ilija and Cevher, Volkan},
  booktitle={2019 IEEE International Symposium on Information Theory (ISIT)},
  pages={2544--2548},
  year={2019},
  organization={IEEE}
}

@inproceedings{moss-audibert2009minimax,
  title={Minimax policies for adversarial and stochastic bandits},
  author={Audibert, Jean-Yves and Bubeck, S{\'e}bastien},
  booktitle={Colt},
  pages={217--226},
  year={2009}
}

@article{harper2015movielens,
  title={The movielens datasets: History and context},
  author={Harper, F Maxwell and Konstan, Joseph A},
  journal={Acm transactions on interactive intelligent systems (tiis)},
  volume={5},
  number={4},
  pages={1--19},
  year={2015},
  publisher={Acm New York, NY, USA}
}

@inproceedings{garivier2016maximin,
  title={Maximin action identification: A new bandit framework for games},
  author={Garivier, Aur{\'e}lien and Kaufmann, Emilie and Koolen, Wouter M},
  booktitle={Conference on Learning Theory},
  pages={1028--1050},
  year={2016},
  organization={PMLR}
}

@inproceedings{maxmin-grouped-wang2022max,
  title={Max-min grouped bandits},
  author={Wang, Zhenlin and Scarlett, Jonathan},
  booktitle={Proceedings of the AAAI Conference on Artificial Intelligence},
  volume={36},
  number={8},
  pages={8603--8611},
  year={2022}
}

@article{shahverdikondori2025best,
  title={Best Group Identification in Multi-Objective Bandits},
  author={Shahverdikondori, Mohammad and Badri, Mohammad Reza and Kiyavash, Negar},
  journal={arXiv preprint arXiv:2505.17869},
  year={2025}
}

@inproceedings{budgeted2-maiti2022causal,
  title={A causal bandit approach to learning good atomic interventions in presence of unobserved confounders},
  author={Maiti, Aurghya and Nair, Vineet and Sinha, Gaurav},
  booktitle={Uncertainty in Artificial Intelligence},
  pages={1328--1338},
  year={2022},
  organization={PMLR}
}

@article{shahverdikondori2025graph,
  title={Graph Learning is Suboptimal in Causal Bandits},
  author={Shahverdikondori, Mohammad and Etesami, Jalal and Kiyavash, Negar},
  journal={arXiv preprint arXiv:2510.16811},
  year={2025}
}

@inproceedings{zhao2023revisiting,
  title={Revisiting simple regret: Fast rates for returning a good arm},
  author={Zhao, Yao and Stephens, Connor and Szepesv{\'a}ri, Csaba and Jun, Kwang-Sung},
  booktitle={International Conference on Machine Learning},
  pages={42110--42158},
  year={2023},
  organization={PMLR}
}

@inproceedings{bubeck2009pure,
  title={Pure exploration in multi-armed bandits problems},
  author={Bubeck, S{\'e}bastien and Munos, R{\'e}mi and Stoltz, Gilles},
  booktitle={International conference on Algorithmic learning theory},
  pages={23--37},
  year={2009},
  organization={Springer}
}

@article{liu2026efficient,
  title={Efficient Simple Regret Algorithms for Stochastic Contextual Bandits},
  author={Liu, Shuai and Bakhtiari, Alireza and Ayoub, Alex and Hao, Botao and Szepesv{\'a}ri, Csaba},
  journal={arXiv preprint arXiv:2601.21167},
  year={2026}
}


\newpage

\onecolumn

\appendix

\begin{center}
    {\LARGE \bfseries Appendix}
\end{center}

The appendix is organized as follows. Section~\ref{apx:related_work} provides a more detailed discussion of related work. 
Section~\ref{apx: proofs} contains the proofs omitted from the main text. Finally, Section~\ref{apx:exp} presents additional experimental details and results.

\section{Further Discussion on Related Work}\label{apx:related_work}

In this section, we provide some additional discussion of related work.











\textbf{Heterogeneous Treatment Effects and Subpopulation-Aware Decision Making.}
A central line of work studies how treatment heterogeneity across subpopulations fundamentally alters optimal experimental design and decision rules. \cite{russac2021b} shows that ignoring heterogeneity across observed or latent subgroups can lead to systematically suboptimal or even misleading treatment selection, even when each arm is evaluated against a control. When outcomes vary across groups, the problem shifts from selecting a globally best arm to learning \textit{group-conditional policies}, effectively turning A/B/n testing into a policy learning problem. This perspective is further developed in adaptive clinical trial settings by \cite{curth2023adaptive}, who propose methods to sequentially identify subpopulations with positive treatment effects under strict statistical constraints. Their work emphasizes adaptive refinement of target populations via procedures such as AdaGGI and AdaGCPI, focusing not on a single optimal subgroup but on any collection of beneficiaries.

Complementing this line, \cite{wu2023best} and \cite{russo2024fair} study best arm identification under \textit{fairness} constraints across subpopulations. Unlike \cite{russac2021b}, which allows fully personalized policies, fairness requirements couple subpopulations by imposing performance constraints across groups. This fundamentally changes the decision problem: the goal is no longer purely optimal identification but balancing optimality with worst-case guarantees across groups, at the cost of increased sample complexity.

\textbf{Contextual and Policy Learning Bandits.}
A related but distinct literature incorporates observable covariates into sequential decision making. \cite{perchet2013multi} extends the classical bandit setting to contextual information, where the goal is to learn a context-dependent policy rather than a single best arm. Under smoothness assumptions, they show that partition-based methods achieve minimax optimal regret by learning locally optimal actions within regions of the covariate space.

In the small gap regime, \cite{kato2022best} further studies best arm identification with context, showing that contextual structure can significantly reduce sample complexity when arms are hard to distinguish. The key difference is that efficiency gains arise only when contexts meaningfully separate arms; otherwise, learning remains intrinsically difficult.
More generally, \cite{kato2024adaptive} shows that even in contextual settings, standard experimental designs are suboptimal when the goal is policy learning rather than estimation, and proposes adaptive designs that concentrate sampling in regions where policy decisions are most uncertain. Similarly, \cite{contextual-simple-regret-deshmukh2018simple,jorke2022simple} frame simple regret minimization in contextual bandits as a Bayesian experimental design problem where exploration is guided by expected information gain about the optimal policy rather than reward maximization.


Beyond bandits, Bayesian optimal experimental design provides a general framework for sequential data acquisition. \cite{foster2019variational} develops scalable variational approximations to expected information gain, enabling BOED in complex models where exact computation is infeasible. These methods connect naturally to contextual bandits and policy learning approaches that also select experiments based on information gain rather than immediate reward.

A complementary perspective arises in supervised learning under covariate shift which could be the result of different subpopulations. \cite{bickel2009discriminative} shows that distribution mismatch between training and test covariates biases standard empirical risk minimization, and proposes density-ratio reweighting to correct this issue. While outside the sequential decision-making setting, this line of work shares with contextual bandits and adaptive design the central idea that data collection and inference must be aligned with the target decision environment.

\textbf{Bandits with Grouped Arms.}
Our setting is also related to bandit problems with grouped arms. Indeed, each subpopulation can be viewed as a group of $n$ treatment arms, so the learner effectively faces $k$ groups and seeks to identify the best arm within each group. The closest line of work is multi-bandit best arm identification \cite{multibandit-bai-gabillon2011multi, multi-bandit-bai2-scarlett2019overlapping}, which studies identifying the best arm in each of several bandits with the objective of minimizing the probability of misidentification. This is particularly close to our fully active setting, since both problems require allocating samples across several independent bandit instances. However, those works focus on best arm identification, whereas our objective is to minimize context-weighted simple regret.

Other grouped arm formulations consider different objectives. Maximin action identification \cite{garivier2016maximin} and max-min grouped bandits \cite{maxmin-grouped-wang2022max} study settings in which performance is determined by the worst arm or outcome within a group, while best group identification in multi-objective bandits \cite{shahverdikondori2025best} focuses on identifying a single best group under multi-objective rewards. In contrast, our goal is not to identify one best group or action, but to recommend the best arm separately within every group. In addition, we characterize the gain obtained by actively choosing which group, that is, which subpopulation, to sample.

\textbf{Causal Bandits.}
Another related line of work is the general framework of causal bandits \cite{lattimore2016causal, pomis-lee2018structural, CUCB-lu2020regret, jamshidi2024confounded}, where actions correspond to interventions on subsets of variables in a causal graph and the learner observes the reward together with the non-intervened variables; see \cite{lattimore2016causal} for details. Our setting can be viewed as a simple causal bandit problem with three nodes: a treatment node, a subpopulation node, and a reward node, where the reward is affected by both the treatment and the subpopulation, while the two parent variables are independent. Under this interpretation, the passive setting corresponds to intervening only on the treatment node, whereas the active setting corresponds to intervening on both the treatment and subpopulation nodes.

A related class of models in this literature is the \emph{parallel graph} setting, in which several independent variables all directly affect the reward \cite{lattimore2016causal, budget-nair2021budgeted}. In \cite{lattimore2016causal} the authors provide algorithms with simple regret guarantees for binary variables, while \cite{budget-nair2021budgeted} and \cite{budgeted2-maiti2022causal} study more general no-backdoor and general graphs and settings with intervention costs, and derive both simple and cumulative regret guarantees. Another related work studies best arm identification in a setting where the context is not independent of the treatment but is instead a random function of the chosen treatment \cite{shahverdikondori2025optimal}.

These works are structurally related to ours, but differ substantially in objective and in the type of interventions being compared. In particular, they focus on identifying a single globally best action, whereas our goal is to identify the best treatment separately for each subpopulation and to characterize the gain obtained by actively choosing both the treatment and the subpopulation, rather than only one of them. Another recent work has also studied the effect of the number of intervened nodes at each round, but the results are on cumulative regret in a setting where the causal structure is unknown \cite{shahverdikondori2025graph}. Extending our results to settings with multiple variables and richer causal structures would be an interesting future direction.

\section{Omitted Proofs} \label{apx: proofs}

This section presents the proofs omitted from the main text due to space constraints. We first start with a Lemma that will be used throughout the proofs.

\begin{lemma}[Concentration of Independent Bernoulli Counts]\label{lem:independent-bernoulli-count}
Let $X_1,\dots,X_m$ be independent Bernoulli random variables with means $\alpha_1,\dots,\alpha_m \in [0,1]$. Define
$$
S_m \coloneqq \sum_{t=1}^m X_t,
\qquad
M_m \coloneqq \ebb[S_m] = \sum_{t=1}^m \alpha_t.
$$
If, for some $\eta \in (0,1)$,
$$
M_m \geq 24 \ln\paran{\frac{2}{\eta}},
$$
then
$$
\pr\paran{\frac{M_m}{2} \leq S_m \leq \frac{3M_m}{2}} \geq 1-\eta.
$$
\end{lemma}

\begin{proof}
Since $X_1,\dots,X_m$ are independent Bernoulli random variables, possibly with different means, the multiplicative Chernoff bounds apply; see, e.g., \cite{dubhashi2009concentration}. In particular,
\begin{align*}
    \pr\paran{S_m < \frac{M_m}{2}}
    &\leq
    \exp\paran{-\frac{M_m}{8}}, \\
    \pr\paran{S_m > \frac{3M_m}{2}}
    &\leq
    \exp\paran{-\frac{M_m}{12}}.
\end{align*}
Therefore, by a union bound,
\begin{align*}
    \pr\paran{\frac{M_m}{2} \leq S_m \leq \frac{3M_m}{2}}
    &\geq
    1-\exp\paran{-\frac{M_m}{8}}-\exp\paran{-\frac{M_m}{12}} \\
    &\geq
    1-2\exp\paran{-\frac{M_m}{24}}.
\end{align*}
Since $M_m \geq 24 \ln(2/\eta)$, it follows that
$$
2\exp\paran{-\frac{M_m}{24}} \leq \eta,
$$
and hence
$$
\pr\paran{\frac{M_m}{2} \leq S_m \leq \frac{3M_m}{2}} \geq 1-\eta.
$$
\end{proof}

Next, we introduce inequalities on the norms of vectors in $\mathbb{R}^k$. These inequalities are standard; see, for example, \cite[Chapter~2]{hardy1952inequalities}.

\begin{theorem}[Comparison and Interpolation of $\ell_p$ Norms]\label{thm:lp-comparison-interpolation}
For any vector $x \in \mathbb{R}^k$ and any exponents $0 < p < q < \infty$, we have
$$
\lpnorm{x}{q} \leq \lpnorm{x}{p} \leq k^{\frac{1}{p}-\frac{1}{q}} \lpnorm{x}{q}.
$$
More generally, for any $0 < p \leq r \leq q < \infty$ and any $\theta \in [0,1]$ such that
$$
\frac{1}{r} = \frac{\theta}{p} + \frac{1-\theta}{q},
$$
we have
$$
\lpnorm{x}{r} \leq \lpnorm{x}{p}^{\theta}\lpnorm{x}{q}^{1-\theta}.
$$
\end{theorem}

\subsection{Proofs of Section \ref{sec:known_p}}

\subsubsection{Proof of Lemma \ref{lem: lower bound}}

\rewardfreeOptimal*

\begin{proof}
Fix $\mathbf p \in \delkplus$ and an arbitrary policy $\pi \in \Pi(\mathbf p)$.

For each subpopulation $j \in [k]$ and arm $a \in [n]$, let $T_{a,j}$ denote the number of times arm $a$ is pulled in subpopulation $j$ over the $T$ rounds. For a family of instances defined below and indexed by $\omega$, let
$$
\bar T_{a,j} \coloneqq \frac{1}{2^k} \sum_{\omega \in \{0,1\}^k} \ebb_{\omega}[T_{a,j}],
\qquad
\bar T_j \coloneqq \sum_{a \in [n]} \bar T_{a,j},
$$
where $\ebb_{\omega}$ denotes expectation under instance $V^{\omega}$. For each $j$, choose
$$
b_j \in \argmin_{a \neq 1} \bar T_{a,j}.
$$
Then, by the pigeonhole principle,
$$
\bar T_{b_j,j} \leq \frac{\bar T_j}{n-1}.
$$

Now fix gaps $\delta_1,\dots,\delta_k \in (0,1/4]$, and consider the family of instances
$$
\mathcal V \coloneqq \{V^{\omega} : \omega \in \{0,1\}^k\},
$$
where rewards are Bernoulli and, for each $\omega = (\omega_1,\dots,\omega_k)$, the mean rewards are given by
$$
\mu_{a,j}^{\omega} =
\begin{cases}
\frac12 + \delta_j, & \text{if } a = 1,\\
\frac12 + 2\delta_j \omega_j, & \text{if } a = b_j,\\
\frac12, & \text{otherwise.}
\end{cases}
$$
Thus, in subpopulation $j$, the optimal arm is
$$
a_j^*(\omega) =
\begin{cases}
1, & \text{if } \omega_j = 0,\\
b_j, & \text{if } \omega_j = 1.
\end{cases}
$$

For any $\omega \in \{0,1\}^k$, let $\text{SR}(\pi,T,V^{\omega})$ denote the simple regret of $\pi$ on instance $V^{\omega}$. Since recommending a non-optimal arm in subpopulation $j$ incurs regret at least $\delta_j$, we have
\begin{align}
    \text{SR}(\pi,T,V^{\omega})
    \geq
    \sum_{j \in [k]} p_j \delta_j \pr_{\omega}\paran{\ahatj \neq a_j^*(\omega)},
    \label{eq:rf-proof-inst-lb}
\end{align}
where $\pr_{\omega}$ denotes the probability measure induced by the full interaction between $\pi$ and $V^{\omega}$.

Averaging over the family and using the fact that the worst-case regret dominates the average regret, we obtain
\begin{align}
    \srtpi
    &\geq \frac{1}{2^k} \sum_{\omega \in \{0,1\}^k} \text{SR}(\pi,T,V^{\omega}) \nonumber\\
    &\geq \sum_{j \in [k]} p_j \delta_j \bar e_j,
    \label{eq:rf-proof-average}
\end{align}
where
$$
\bar e_j \coloneqq \frac{1}{2^k} \sum_{\omega \in \{0,1\}^k} \pr_{\omega}\paran{\ahatj \neq a_j^*(\omega)}.
$$

We now lower bound $\bar e_j$. Fix $j \in [k]$. For $\omega$ with $\omega_j = 0$, let $\omega^{\lnot j}$ denote the vector obtained by flipping the $j$-th coordinate of $\omega$, and let
$$
B_j \coloneqq \{\ahatj = 1\}.
$$
Under $V^{\omega}$, arm $1$ is optimal in subpopulation $j$, while under $V^{\omega^{\lnot j}}$, arm $b_j$ is optimal. Therefore,
\begin{align*}
    \pr_{\omega}\paran{\ahatj \neq a_j^*(\omega)}
    +
    \pr_{\omega^{\lnot j}}\paran{\ahatj \neq a_j^*(\omega^{\lnot j})}
    \geq
    \pr_{\omega}(B_j^c) + \pr_{\omega^{\lnot j}}(B_j).
\end{align*}
By the Bretagnolle--Huber inequality,
\begin{align*}
    \pr_{\omega}(B_j^c) + \pr_{\omega^{\lnot j}}(B_j)
    \geq
    \frac12 \exp\paran{-D(\pr_{\omega},\pr_{\omega^{\lnot j}})},
\end{align*}

where $D(\cdot, \cdot)$ shows the KL divergence.
Hence,
\begin{align}
    \pr_{\omega}\paran{\ahatj \neq a_j^*(\omega)}
    +
    \pr_{\omega^{\lnot j}}\paran{\ahatj \neq a_j^*(\omega^{\lnot j})}
    \geq
    \frac12 \exp\paran{-D(\pr_{\omega},\pr_{\omega^{\lnot j}})}.
    \label{eq:rf-proof-bh}
\end{align}

Averaging \eqref{eq:rf-proof-bh} over all $\omega$ with $\omega_j = 0$ gives
\begin{align}
    \bar e_j
    \geq
    \frac{1}{2^{k+1}}
    \sum_{\omega : \omega_j = 0}
    \exp\paran{-D(\pr_{\omega},\pr_{\omega^{\lnot j}})}.
    \label{eq:rf-proof-ej1}
\end{align}

We now control the KL terms. Since $V^{\omega}$ and $V^{\omega^{\lnot j}}$ differ only in the reward distribution of arm $b_j$ in subpopulation $j$, by KL divergence decomposition Lemma \cite{bandit-book1-lattimore2020bandit} we have
\begin{align*}
    D(\pr_{\omega},\pr_{\omega^{\lnot j}})
    =
    \ebb_{\omega}[T_{b_j,j}] \,
    \kl{\frac12,\frac12+2\delta_j},
\end{align*}
where $\kl{\cdot,\cdot}$ denotes the Bernoulli KL divergence. Since $\delta_j \leq 1/4$, there exists a universal constant $c_0>0$ such that
$$
\kl{\frac12,\frac12+2\delta_j} \leq c_0 \delta_j^2,
$$
and therefore
\begin{align*}
    D(\pr_{\omega},\pr_{\omega^{\lnot j}})
    \leq
    c_0 \ebb_{\omega}[T_{b_j,j}] \delta_j^2.
\end{align*}
Using Jensen's inequality in \eqref{eq:rf-proof-ej1}, we obtain
\begin{align*}
    \bar e_j
    &\geq
    \frac14
    \exp\paran{
        - \frac{c_0 \delta_j^2}{2^{k-1}}
        \sum_{\omega:\omega_j=0} \ebb_{\omega}[T_{b_j,j}]
    }.
\end{align*}
Since
$$
\frac{1}{2^{k-1}} \sum_{\omega:\omega_j=0} \ebb_{\omega}[T_{b_j,j}]
\leq
2 \bar T_{b_j,j}
\leq
\frac{2\bar T_j}{n-1},
$$
it follows that
\begin{align}
    \bar e_j
    \geq
    \frac14 \exp\paran{- \frac{2c_0 \bar T_j}{n-1}\delta_j^2 }.
    \label{eq:rf-proof-ej2}
\end{align}

Substituting \eqref{eq:rf-proof-ej2} into \eqref{eq:rf-proof-average} gives
\begin{align}
    \srtpi
    \geq
    \frac14 \sum_{j \in [k]} p_j \delta_j
    \exp\paran{- \frac{2c_0 \bar T_j}{n-1}\delta_j^2 }.
    \label{eq:rf-proof-main-lb}
\end{align}

We now optimize each term in \eqref{eq:rf-proof-main-lb} with respect to $\delta_j$. For each $j \in [k]$, we set
$$
\delta_j
=
\sqrt{\frac{n-1}{4c_0 \bar T_j}}.
$$
Since we work in the data-rich regime and $k$ and $1/\pmin$ are treated as constants, we may assume $T$ is large enough so that $\delta_j \leq 1/4$ for all $j$. Plugging the value of $\delta_j$ into \eqref{eq:rf-proof-main-lb}, we obtain
\begin{align}
    \srtpi
    \geq
    c_1 \sqrt{n-1} \sum_{j \in [k]} \frac{p_j}{\sqrt{\bar T_j}}
    \geq
    c_2 \sqrt{\frac{n}{T}} \sum_{j \in [k]} \frac{p_j}{\sqrt{\bar q_j}},
    \label{eq:rf-proof-q-lb}
\end{align}
for universal constants $c_1,c_2>0$, where
$$
\bar q_j \coloneqq \frac{\bar T_j}{T},
\qquad
\sum_{j \in [k]} \bar q_j = 1.
$$

It remains to optimize the right hand side of \eqref{eq:rf-proof-q-lb} over all $\bar {\mb{q}} \in \delkplus$. 

In the proof of Lemma \ref{lem:optimal_q_sp}, we show that the minimum of this optimization problem is achieved at the unique point $\bar{\mb{q}}^*$ with 
$$
\bar q^*_j \propto p_j ^ {2/3},
$$

and the optimal value is $\lpnorm{\mb{p}}{2/3}$.

Therefore, \eqref{eq:rf-proof-q-lb} yields
\begin{align}
    \inf_{\pi \in \Pi(\mathbf p)} \srtpi
    \in
    \Omega\paran{\sqrt{\frac{n}{T}} \, \lpnorm{\mathbf p}{2/3}},
    \label{eq:rf-proof-lower}
\end{align}

which completes the proof.


\end{proof}

\subsubsection{Proof of Lemma \ref{lem:history-free simple regret}} 

\rewardfreeSimpleRegret*

\begin{proof}
For a history-free policy $\pi \in \Piind{hf}(\mathbf p)$, let
$$
T_j \coloneqq \sum_{t=1}^T \mb{1}\{C_t=j\}
$$
denote the number of rounds in which subpopulation $j$ is observed. Since the subpopulation selection rule is history-free, there exist deterministic probabilities $\{q_{t,j}\}_{t \in [T],\, j \in [k]}$, depending only on $\mathbf p$, such that
$$
\pr(C_t=j) = q_{t,j},
\qquad
\sum_{j \in [k]} q_{t,j} = 1
\quad \forall t \in [T].
$$
Hence $T_j$ is a sum of independent Bernoulli random variables, and
$$
\ebb[T_j] = \sum_{t=1}^T q_{t,j} = T \qjpip.
$$

We first prove the lower bound. We use the lower bound construction from the proof of Lemma~\ref{lem: lower bound}. In that proof, for an arbitrary policy $\pi$, one constructs a family of hard instances and derives the bound
$$
\srtpi
\in
\Omega\paran{\sqrt{\frac{n}{T}} \sum_{j \in [k]} \frac{p_j}{\sqrt{\bar q_j}}},
$$
where $\bar q_j$ denotes the average fraction of rounds allocated to subpopulation $j$ across the hard instance family. For a history-free policy, the subpopulation selection rule is independent of the observed history and therefore independent of the realized rewards and of the instance itself. Consequently, the allocation vector is the same on every instance in the family, and thus
$$
\bar q_j = \qjpip
\qquad \forall j \in [k].
$$
Substituting this into the lower bound from Lemma~\ref{lem: lower bound} yields
\begin{align}
    \srtpi
    \in
    \Omega\paran{\sqrt{\frac{n}{T}} \sum_{j \in [k]} \frac{p_j}{\sqrt{\qjpip}} }.
    \label{eq:hf-lower}
\end{align}

We now prove the upper bound. For each $j \in [k]$, define the event
$$
\mc{E}_j \coloneqq \brak{\frac{T\qjpip}{2} \leq T_j \leq \frac{3T\qjpip}{2}},
$$
and let
$$
\mc{E} \coloneqq \bigcap_{j \in [k]} \mc{E}_j.
$$
Since $T_j$ is a sum of independent Bernoulli random variables with mean $T\qjpip$, Lemma~\ref{lem:independent-bernoulli-count} with $\eta = 1/k T$ implies that
$$
\pr(\mc{E}_j) \geq 1-\frac{1}{kT},
$$
where we use the assumption $T\qjpip > 24 \ln(2kT)$.

Therefore, by a union bound,
\begin{align}
    \pr(\mc{E}^c)
    \leq
    \sum_{j \in [k]} \pr(\mc{E}_j^c)
    \leq
    \frac{1}{T}.
    \label{eq:hf-event-fail}
\end{align}

Fix any instance $(\bnu,\mathbf p)$. On the event $\mc{E}$, each subpopulation $j$ is observed at least $T\qjpip/2$ times. Since the subroutine $\algsr$ is anytime and achieves the optimal simple regret rate in Lemma~\ref{lem:simple_regret}, there exists a universal constant $c_1>0$ such that the expected recommendation gap in subpopulation $j$ is at most
$$
c_1 \sqrt{\frac{n}{T_j}}
\leq
c_1 \sqrt{\frac{2n}{T\qjpip}}.
$$
Therefore, on $\mc{E}$,
\begin{align*}
    \text{SR}(\pi,T,(\bnu,\mathbf p))
    &=
    \sum_{j \in [k]} p_j \Delta_j \\
    &\leq
    c_1 \sqrt{\frac{2n}{T}} \sum_{j \in [k]} \frac{p_j}{\sqrt{\qjpip}}.
\end{align*}

On the complement event $\mc{E}^c$, the simple regret is at most $1$, since rewards lie in $[0,1]$ and $\sum_j p_j = 1$. Hence
\begin{align*}
    \text{SR}(\pi,T,(\bnu,\mathbf p))
    &\leq
    c_1 \sqrt{\frac{2n}{T}} \sum_{j \in [k]} \frac{p_j}{\sqrt{\qjpip}}
    + \pr(\mc{E}^c).
\end{align*}
Taking the supremum over all instances and using \eqref{eq:hf-event-fail}, we obtain
\begin{align*}
    \srtpi
    &\leq
    c_1 \sqrt{\frac{2n}{T}} \sum_{j \in [k]} \frac{p_j}{\sqrt{\qjpip}}
    + \frac{1}{T}.
\end{align*}
Finally, since $\qjpip \leq 1$ for all $j$, we have
$$
\sum_{j \in [k]} \frac{p_j}{\sqrt{\qjpip}}
\geq
\sum_{j \in [k]} p_j
=
1,
$$
and therefore the term $1/T$ is of smaller order than the main term. Thus,
\begin{align}
    \srtpi
    \in
    O\paran{\sqrt{\frac{n}{T}} \sum_{j \in [k]} \frac{p_j}{\sqrt{\qjpip}} }.
    \label{eq:hf-upper}
\end{align}

Combining \eqref{eq:hf-lower} and \eqref{eq:hf-upper} proves the claim.
\end{proof}

\subsubsection{Proof of Lemma \ref{lem:optimal_q_sp}}

\optimalQSp*

\begin{proof}
Recall that
$$
S_{\mathbf p}(\mathbf q) = \sum_{j \in [k]} \frac{p_j}{\sqrt{q_j}},
\qquad
\mathbf q \in \delkplus.
$$
Since $p_j>0$ for all $j \in [k]$, each term $p_j q_j^{-1/2}$ is strictly convex on $(0,\infty)$. Therefore, $S_{\mathbf p}$ is strictly convex on $\delkplus$. Hence, if a minimizer exists, it is unique.

We now characterize this minimizer. Since $p_j>0$ for all $j$, we have
$$
\frac{p_j}{\sqrt{q_j}} \to +\infty
\qquad \text{as } q_j \downarrow 0.
$$
Therefore, the minimum over the simplex cannot be attained on the boundary, and the minimizer must lie in the interior. Thus, we may use the first order optimality conditions for the equality-constrained problem
\begin{align*}
    \min_{\mathbf q \in \mathbb{R}^k} \quad & \sum_{j \in [k]} \frac{p_j}{\sqrt{q_j}} \\
    \text{s.t.} \quad & \sum_{j \in [k]} q_j = 1, \qquad q_j > 0 \ \ \forall j \in [k].
\end{align*}

Consider the Lagrangian
$$
\mathcal L(\mathbf q,\lambda)
=
\sum_{j \in [k]} p_j q_j^{-1/2}
+
\lambda \paran{\sum_{j \in [k]} q_j - 1}.
$$
At the optimum, for each $j \in [k]$,
$$
\frac{\partial \mathcal L}{\partial q_j}
=
-\frac{1}{2} p_j q_j^{-3/2} + \lambda
=
0.
$$
Hence
$$
q_j^{-3/2} = \frac{2\lambda}{p_j},
$$
which implies
$$
q_j = \paran{\frac{p_j}{2\lambda}}^{2/3}.
$$
Therefore, there exists a constant $c>0$ such that
$$
q_j = c\, p_j^{2/3}
\qquad \forall j \in [k].
$$
Using the constraint $\sum_{j \in [k]} q_j = 1$, we obtain
$$
1 = \sum_{j \in [k]} q_j = c \sum_{j \in [k]} p_j^{2/3},
$$
and thus
$$
c = \frac{1}{\sum_{l \in [k]} p_l^{2/3}}.
$$
Substituting back gives
$$
q_j^*(\mathbf p)
=
\frac{p_j^{2/3}}{\sum_{l \in [k]} p_l^{2/3}},
\qquad j \in [k].
$$

Since $S_{\mathbf p}$ is strictly convex on $\delkplus$, this minimizer is unique.

Then, by substituting this back to $S_{\mb{p}}(\mb{q})$, we obtain
$$
S_{\mathbf p}(\mathbf q^*(\mathbf p)) = 
\paran{\sum_{j \in [k]} p_j ^ {2/3} } ^ {3/2}=
 \lpnorm{\mathbf p}{2/3}.
$$
\end{proof}

\subsubsection{Proof of Lemma \ref{lem:active_passive_gap}}

\activePassiveGap*

\begin{proof}

Recall that
$$
R(\mathbf p) = \frac{\lpnorm{\mathbf p}{1/2}^{1/2}}{\lpnorm{\mathbf p}{2/3}}.
$$

We first show that $R(\mathbf p) \geq 1$ for every $\mathbf p \in \delkplus$. Let $a_j \coloneqq \sqrt{p_j}$ for all $j \in [k]$, and let $\mathbf a = (a_1,\dots,a_k)$. Then
$$
\lpnorm{\mathbf a}{1} = \sum_{j \in [k]} \sqrt{p_j} = \lpnorm{\mathbf p}{1/2}^{1/2},
\qquad
\lpnorm{\mathbf a}{2}^2 = \sum_{j \in [k]} p_j = 1,
$$
and
$$
\lpnorm{\mathbf a}{4/3}^2
=
\paran{\sum_{j \in [k]} p_j^{2/3}}^{3/2}
=
\lpnorm{\mathbf p}{2/3}.
$$

Applying the interpolation part of Theorem~\ref{thm:lp-comparison-interpolation} with $p=1$, $r=4/3$, $q=2$, and $\theta=1/2$, we obtain
$$
\lpnorm{\mathbf a}{4/3} \leq \lpnorm{\mathbf a}{1}^{1/2}\lpnorm{\mathbf a}{2}^{1/2}.
$$
Since $\lpnorm{\mathbf a}{2}=1$, this yields
$$
\lpnorm{\mathbf p}{2/3} = \lpnorm{\mathbf a}{4/3}^2 \leq \lpnorm{\mathbf a}{1} = \lpnorm{\mathbf p}{1/2}^{1/2},
$$
and hence $R(\mathbf p) \geq 1$.

For the upper bound, applying the comparison part of Theorem~\ref{thm:lp-comparison-interpolation} with $p=1/2$ and $q=2/3$ gives
$$
\lpnorm{\mathbf p}{1/2} \leq k^{1/2}\lpnorm{\mathbf p}{2/3}.
$$
Taking square roots and using $\lpnorm{\mathbf p}{2/3} \geq \lpnorm{\mathbf p}{1} = 1$, we obtain
$$
\lpnorm{\mathbf p}{1/2}^{1/2} \leq k^{1/4}\lpnorm{\mathbf p}{2/3},
$$
which implies $R(\mathbf p) \leq k^{1/4}$.

It remains to prove the matching lower bound. For $k=1$, the claim is trivial, so assume $k \geq 2$. Consider the family of probability vectors
$$
\mathbf p^{(\varepsilon)}
=
\paran{1-\varepsilon,\frac{\varepsilon}{k-1},\dots,\frac{\varepsilon}{k-1}},
\qquad
\varepsilon \in (0,1).
$$
Choose
$$
\varepsilon = \frac{1}{\sqrt{k-1}}.
$$
Then
\begin{align*}
    \lpnorm{\mathbf p^{(\varepsilon)}}{1/2}^{1/2}
    =
    \sqrt{1-\varepsilon}
    +
    (k-1)\sqrt{\frac{\varepsilon}{k-1}} 
    =
    \sqrt{1-\varepsilon}
    +
    \sqrt{\varepsilon(k-1)}
    \geq
    (k-1)^{1/4}.
\end{align*}
On the other hand,
\begin{align*}
    \lpnorm{\mathbf p^{(\varepsilon)}}{2/3}
    &=
    \paran{
        (1-\varepsilon)^{2/3}
        +
        (k-1)\paran{\frac{\varepsilon}{k-1}}^{2/3}
    }^{3/2} \\
    &=
    \paran{
        (1-\varepsilon)^{2/3}
        +
        (k-1)^{1/3}\varepsilon^{2/3}
    }^{3/2}.
\end{align*}
With $\varepsilon = (k-1)^{-1/2}$, we have
$$
(k-1)^{1/3}\varepsilon^{2/3} = 1,
$$
and therefore
$$
\lpnorm{\mathbf p^{(\varepsilon)}}{2/3}
=
\paran{
    (1-\varepsilon)^{2/3}+1
}^{3/2}
\leq
2^{3/2}
=
\sqrt{8}.
$$
Combining the two bounds yields
$$
R\paran{\mathbf p^{(\varepsilon)}}
=
\frac{\lpnorm{\mathbf p^{(\varepsilon)}}{1/2}^{1/2}}{\lpnorm{\mathbf p^{(\varepsilon)}}{2/3}}
\geq
\frac{(k-1)^{1/4}}{\sqrt{8}}.
$$
Hence
$$
\sup_{\mathbf p \in \delkplus} R(\mathbf p) \in \Omega\paran{k^{1/4}}.
$$

Combining the upper and lower bounds, we conclude that
$$
\sup_{\mathbf p \in \delkplus} R(\mathbf p) \in \Theta\paran{k^{1/4}},
$$
and this order is tight.
\end{proof}

\subsubsection{Proof of Theorem \ref{thm:budget}}

\budget*

\begin{proof}
For simplicity, assume that $\alpha T$ is an integer. Replacing $\alpha T$ by $\lfloor \alpha T \rfloor$ changes the number of passive and active rounds by at most one and therefore does not affect the regret rate.

We first characterize the optimizer of
\begin{align}
    \label{eq:budget-opt-proof}
    \min_{\mathbf q \in \delkplus}
    \quad &
    S_{\mathbf p}(\mathbf q)
    =
    \sum_{j \in [k]} \frac{p_j}{\sqrt{q_j}}
    \\
    \text{s.t.} \quad &
    q_j \geq (1-\alpha)p_j, \qquad \forall j \in [k].
    \nonumber
\end{align}

When $\alpha=0$, the constraints reduce to $q_j \geq p_j$ for all $j$, and since $\sum_j q_j = \sum_j p_j = 1$, the feasible set contains only the point $\mathbf q=\mathbf p$. Hence $\qstalphap=\mathbf p$, and the claim is immediate. We therefore assume in the remainder of the proof that $\alpha \in (0,1]$.

Since $\delkplus$ is not compact, we instead consider the same problem over the closed simplex $\delk$. The feasible set
$$
\mathcal Q_\alpha
\coloneqq
\brak{\mathbf q \in \delk : q_j \geq (1-\alpha)p_j,\ \forall j \in [k]}
$$
is compact and nonempty, so the minimum is attained on $\mathcal Q_\alpha$. Moreover, because each $p_j>0$,
$$
\frac{p_j}{\sqrt{q_j}} \to +\infty
\qquad \text{as } q_j \downarrow 0.
$$
Hence the minimizer cannot lie on the boundary $\{q_j=0\}$, and therefore lies in $\delkplus$. Thus the minimizer over $\mathcal Q_\alpha$ is also the minimizer of \eqref{eq:budget-opt-proof}.

The objective is strictly convex on $\delkplus$, since each map $q_j \mapsto p_j q_j^{-1/2}$ is strictly convex on $(0,\infty)$. Hence the minimizer is unique.

We now derive its form. Let
$$
b_j \coloneqq (1-\alpha)p_j,
\qquad j \in [k].
$$
By the KKT conditions, there exist $\lambda \in \mathbb{R}$ and multipliers $\nu_j \geq 0$ such that
$$
-\frac12 p_j q_j^{-3/2} + \lambda - \nu_j = 0,
\qquad
\nu_j(b_j-q_j)=0,
\qquad
q_j \geq b_j.
$$
Therefore, if $q_j>b_j$, then $\nu_j=0$ and
$$
q_j = \paran{\frac{p_j}{2\lambda}}^{2/3} = c\, p_j^{2/3}
$$
for some constant $c>0$. If $q_j=b_j$, then the lower bound constraint is active. Hence the unique optimizer must be of the form
$$
\qstalphapj = \max\big((1-\alpha)p_j,\; c^* p_j^{2/3}\big),
\qquad j \in [k],
$$
for some constant $c^* \geq 0$.

To determine $c^*$, define
$$
F(c) \coloneqq \sum_{j \in [k]} \max\big((1-\alpha)p_j,\; c p_j^{2/3}\big).
$$
Then $F$ is continuous and nondecreasing on $[0,\infty)$. Moreover, since $\alpha>0$,
$$
F(0) = \sum_{j \in [k]} (1-\alpha)p_j = 1-\alpha < 1,
$$
and clearly $F(c)\to\infty$ as $c\to\infty$. Hence there exists at least one $c^*\geq 0$ such that $F(c^*)=1$.

We now show uniqueness. Let $p_{\min} \coloneqq \min_j p_j$. If $c < p_{\min}^{1/3}$, then for every $j \in [k]$,
$$
c p_j^{2/3} < p_{\min}^{1/3} p_j^{2/3} \leq p_j,
\qquad
(1-\alpha)p_j < p_j,
$$
and therefore
$$
\max\big((1-\alpha)p_j,\; c p_j^{2/3}\big) < p_j.
$$
Summing over $j$ gives $F(c)<1$ for all $c<p_{\min}^{1/3}$. On the other hand, for $c \geq p_{\min}^{1/3}$, the coordinate corresponding to $p_{\min}$ is already on the linear branch, so $F$ has strictly positive slope and is therefore strictly increasing on $[p_{\min}^{1/3},\infty)$. Since $F(c)<1$ for all $c<p_{\min}^{1/3}$ and $F$ is strictly increasing afterwards, the equation $F(c)=1$ has a unique solution $c^*$.

Thus,
$$
\qstalphapj = \max\big((1-\alpha)p_j,\; c^* p_j^{2/3}\big),
\qquad j \in [k],
$$
is indeed the unique optimizer of \eqref{eq:budget-opt-proof}.

We now show that Algorithm~\ref{algo:budget} has expected subpopulation proportions equal to $\qstalphap$. By construction, the algorithm computes the same constant $c^*$ and sets
$$
q_j^* = \max\big((1-\alpha)p_j,\; c^* p_j^{2/3}\big)
= \qstalphapj,
\qquad
r_j = \frac{q_j^*-(1-\alpha)p_j}{\alpha}.
$$
Since $q_j^* \geq (1-\alpha)p_j$, we have $r_j \geq 0$, and
\begin{align*}
    \sum_{j \in [k]} r_j
    &=
    \frac{\sum_{j \in [k]} q_j^* - (1-\alpha)\sum_{j \in [k]} p_j}{\alpha}
    =
    \frac{1-(1-\alpha)}{\alpha}
    = 1,
\end{align*}
so $\mathbf r \in \delkplus$.

During the first $(1-\alpha)T$ rounds, the algorithm samples subpopulations according to $\mathbf p$, hence the expected number of observations of subpopulation $j$ in this phase is
$$
(1-\alpha)T p_j.
$$
During the remaining $\alpha T$ rounds, the algorithm samples according to $\mathbf r$, so the expected number of observations of subpopulation $j$ in the active phase is
$$
\alpha T r_j
=
\alpha T \cdot \frac{q_j^*-(1-\alpha)p_j}{\alpha}
=
T\big(q_j^*-(1-\alpha)p_j\big).
$$
Therefore the total expected number of observations of subpopulation $j$ is
$$
(1-\alpha)T p_j + T\big(q_j^*-(1-\alpha)p_j\big) = T q_j^*.
$$
Hence the expected subpopulation proportions of Algorithm~\ref{algo:budget} are exactly
$$
\mathbf q(\texttt{Alg.1},\mathbf p) = \mathbf q^* = \qstalphap.
$$

The subpopulation selection rule of the algorithm depends only on $\mathbf p$, $\alpha$, and the deterministic vector $\mathbf r$. It is therefore independent of the observed history, so the algorithm is history-free. Moreover, it is passive for exactly $(1-\alpha)T$ rounds and active for exactly $\alpha T$ rounds, hence it belongs to $\Piind{\alpha}(\mathbf p)$.

We next verify the condition of Lemma~\ref{lem:history-free simple regret}. Since $c^* \geq p_{\min}^{1/3}$ by the argument above, for every $j \in [k]$,
$$
\qstalphapj
=
\max\big((1-\alpha)p_j,\; c^* p_j^{2/3}\big)
\geq
c^* p_j^{2/3}
\geq
p_{\min}^{1/3} p_j^{2/3}
\geq
p_{\min}.
$$
Thus
$$
\min_{j \in [k]} \qstalphapj \geq p_{\min}.
$$
Since we work in the data rich regime and assume $T \gg 1/p_{\min}$, it follows that,
$$
T \qstalphapj > 24k\ln(2T),
\qquad
\forall j \in [k].
$$
We then apply Lemma~\ref{lem:history-free simple regret} with $\qpip=\qstalphap$ to obtain
\begin{align*}
    \text{SR}(\texttt{Alg.1}, T, \mathbf p)
    &\in
    \ocal\Big(
        \sqrt{\frac{n}{T}}
        \sum_{j \in [k]} \frac{p_j}{\sqrt{\qstalphapj}}
    \Big) \\
    &=
    \ocal\Big(
        \sqrt{\frac{n}{T}} \, S_{\mathbf p}(\qstalphap)
    \Big).
\end{align*}
By definition of $\vstalphap$ as the optimal value of \eqref{eq:budget-opt-proof},
$$
S_{\mathbf p}(\qstalphap) = \vstalphap.
$$
Hence
$$
\text{SR}(\texttt{Alg.1}, T, \mathbf p)
\in
\ocal\Big(\sqrt{\frac{n}{T}} \, \vstalphap\Big).
$$
This proves the theorem.
\end{proof}

\subsubsection{Proof of Lemma \ref{lem:alphamin}}

\alphamin*

\begin{proof}
Recall that the optimal value $\vstalphap$ is defined as the minimum of
\begin{align}
    \label{eq:alphamin-proof-budget}
    \min_{\mathbf q \in \delkplus}
    \quad &
    S_{\mathbf p}(\mathbf q)
    =
    \sum_{j \in [k]} \frac{p_j}{\sqrt{q_j}}
    \\
    \text{s.t.} \quad &
    q_j \geq (1-\alpha)p_j, \qquad \forall j \in [k].
    \nonumber
\end{align}
By Lemma~\ref{lem:optimal_q_sp}, the unique minimizer of the unconstrained problem over $\delkplus$ is
$$
q_j^{\mathrm{act}}
\coloneqq
\frac{p_j^{2/3}}{\sum_{\ell \in [k]} p_\ell^{2/3}},
\qquad j \in [k],
$$
and its objective value is
$$
S_{\mathbf p}(\mathbf q^{\mathrm{act}})
=
\lpnorm{\mathbf p}{2/3}.
$$

We claim that $\mathbf q^{\mathrm{act}}$ is feasible for \eqref{eq:alphamin-proof-budget} whenever $\alpha \geq \alphmin$. Indeed, the feasibility condition is
$$
q_j^{\mathrm{act}} \geq (1-\alpha)p_j,
\qquad \forall j \in [k],
$$
which is equivalent to
$$
\frac{p_j^{2/3}}{\sum_{\ell \in [k]} p_\ell^{2/3}}
\geq
(1-\alpha)p_j,
\qquad \forall j \in [k].
$$
Since $p_j>0$, dividing by $p_j$ gives
$$
\frac{p_j^{-1/3}}{\sum_{\ell \in [k]} p_\ell^{2/3}}
\geq
1-\alpha,
\qquad \forall j \in [k].
$$
Equivalently,
$$
1-\alpha
\leq
\min_{j \in [k]}
\frac{p_j^{-1/3}}{\sum_{\ell \in [k]} p_\ell^{2/3}}.
$$
Now the function $x \mapsto x^{-1/3}$ is decreasing on $(0,\infty)$, so
$$
\min_{j \in [k]} p_j^{-1/3} = p_{\max}^{-1/3},
$$
where $p_{\max} = \max_j p_j$. Therefore,
$$
\min_{j \in [k]}
\frac{p_j^{-1/3}}{\sum_{\ell \in [k]} p_\ell^{2/3}}
=
\frac{p_{\max}^{-1/3}}{\sum_{\ell \in [k]} p_\ell^{2/3}}.
$$
Hence the feasibility condition becomes
$$
1-\alpha
\leq
\frac{p_{\max}^{-1/3}}{\sum_{\ell \in [k]} p_\ell^{2/3}},
$$
or equivalently
$$
\alpha
\geq
1-
\frac{p_{\max}^{-1/3}}{\sum_{\ell \in [k]} p_\ell^{2/3}}
=
\alphmin.
$$
Thus, for every $\alpha \in [\alphmin,1]$, the fully active optimizer $\mathbf q^{\mathrm{act}}$ is feasible for the budgeted problem \eqref{eq:alphamin-proof-budget}.

Since \eqref{eq:alphamin-proof-budget} minimizes the same objective $S_{\mathbf p}$ as the fully active problem, but over a smaller feasible set, its optimal value is at least the unconstrained optimum:
$$
\vstalphap \geq \lpnorm{\mathbf p}{2/3}.
$$
On the other hand, when $\alpha \geq \alphmin$, the vector $\mathbf q^{\mathrm{act}}$ is feasible, and therefore
$$
\vstalphap \leq S_{\mathbf p}(\mathbf q^{\mathrm{act}})
= \lpnorm{\mathbf p}{2/3}.
$$
Combining the two inequalities yields
$$
\vstalphap = \lpnorm{\mathbf p}{2/3},
\qquad \forall \alpha \in [\alphmin,1].
$$

The final statement follows immediately from Theorem~\ref{thm:budget}: for every $\alpha \in [\alphmin,1]$, Algorithm~\ref{algo:budget} satisfies
$$
\text{SR}(\texttt{Alg.1},T,\mathbf p)
\in
\ocal\Big(\sqrt{\frac{n}{T}}\,\vstalphap\Big)
=
\ocal\Big(\sqrt{\frac{n}{T}}\,\lpnorm{\mathbf p}{2/3}\Big).
$$
This is exactly the optimal rate of fully active policies, which, by Lemma~\ref{lem: lower bound}, is also the optimal rate over all policies up to constant factors.
\end{proof}

\subsection{Proofs of Section \ref{sec:unknown_p}}

This section presents the omitted proofs of Section \ref{sec:unknown_p} of the main text.

\subsubsection{Proof of Theorem \ref{thm:unknown_p}}

\unknownP*

\begin{proof}
We analyze a slightly modified version of EETC in which the passive phase ends at time $\tau_2-1$, the frozen estimate is
$$
\phfreeze \coloneqq \hat{\mathbf p}(\tau_2-1),
$$
and the active phase starts at round $\tau_2$. Compared to the algorithm as stated, this removes the single wasted passive round at time $\tau_2$. This changes the number of samples collected from each subpopulation by at most one and therefore does not affect the regret rate. We will therefore ignore this one round discrepancy in the remainder of the proof.

To analyze the stopping times, consider an auxiliary passive sequence $(\widetilde C_t)_{t \in [T]}$ of i.i.d.\ draws from $\mathbf p$. Up to time $\tau_2-1$, EETC observes exactly the prefix of this sequence, so both stopping times $\tau_1,\tau_2$ are measurable functions of $(\widetilde C_t)_{t \in [T]}$. For $t \in [T]$ and $j \in [k]$, define
$$
\widetilde N_{t,j} \coloneqq \sum_{s=1}^t \mb{1}\{\widetilde C_s=j\},
\qquad
\hat p_{t,j} \coloneqq \frac{\widetilde N_{t,j}}{t}.
$$

We first show that after $O(\log T)$ passive rounds, the empirical proportions are uniformly accurate. Let
$$
m_T \coloneqq \left\lceil \frac{24 \ln(2kT^2)}{p_{\min}} \right\rceil,
\qquad
p_{\min} \coloneqq \min_{j \in [k]} p_j.
$$
Since $p_{\min}>0$ is fixed, $m_T \in O(\log T)$.

For each fixed $t \geq m_T$ and $j \in [k]$, the count $\widetilde N_{t,j}$ is binomial with mean $tp_j \geq m_T p_{\min} \geq 24\ln(2kT^2)$. Hence, by Lemma~\ref{lem:independent-bernoulli-count},
$$
\pr\paran{\frac{tp_j}{2} \leq \widetilde N_{t,j} \leq \frac{3tp_j}{2}} \geq 1-\frac{1}{kT^2}.
$$
Define the event
$$
\mc{G}_1
\coloneqq
\bigcap_{t=m_T}^{T} \bigcap_{j \in [k]}
\brak{
\frac{tp_j}{2} \leq \widetilde N_{t,j} \leq \frac{3tp_j}{2}
}.
$$
By a union bound over at most $kT$ pairs $(t,j)$,
\begin{align}
    \pr(\mc{G}_1^c)
    \leq
    \sum_{t=m_T}^{T}\sum_{j=1}^{k}\frac{1}{kT^2}
    \leq
    \frac{1}{T}.
    \label{eq:unknown-g1}
\end{align}

We next record the consequences of $\mc{G}_1$.

First, on $\mc{G}_1$, for every $j \in [k]$,
$$
\widetilde N_{m_T,j} \geq \frac{m_T p_j}{2} \geq \frac{m_T p_{\min}}{2} \geq 12\ln(2kT^2) > \ln(kT^2),
$$
and therefore, by the definition of $\tau_1$,
\begin{align}
    \tau_1 \leq m_T \in O(\log T).
    \label{eq:unknown-tau1}
\end{align}

We now show that the frozen estimate $\phfreeze$ is accurate. Define
$$
\beta(\mathbf u)
\coloneqq
1-\alphmininp{\mathbf u}
=
\frac{u_{\max}^{-1/3}}{\sum_{j \in [k]} u_j^{2/3}},
\qquad
u_{\max} \coloneqq \max_j u_j.
$$
We first note that $\beta(\mathbf u)$ admits a uniform positive lower bound over all probability vectors $\mathbf u \in \delk$. Indeed, since $u_{\max} \leq 1$, we have
$$
u_{\max}^{-1/3} \geq 1.
$$
Moreover, since $x \mapsto x^{2/3}$ is concave on $[0,\infty)$, Jensen's inequality gives
$$
\frac{1}{k} \sum_{j \in [k]} u_j^{2/3}
\leq
\paran{\frac{1}{k} \sum_{j \in [k]} u_j}^{2/3}
=
\paran{\frac{1}{k}}^{2/3}.
$$
Multiplying both sides by $k$ yields
$$
\sum_{j \in [k]} u_j^{2/3} \leq k^{1/3}.
$$
Hence, for every $\mathbf u \in \delk$,
\begin{align}
    \beta(\mathbf u) \geq k^{-1/3}.
    \label{eq:unknown-beta-universal}
\end{align}

Since $m_T = O(\log T)$, for sufficiently large $T$ we have
$$
\frac{m_T}{T} < k^{-1/3}.
$$
Combining this with \eqref{eq:unknown-beta-universal}, we obtain
$$
\frac{t}{T} \leq \frac{m_T}{T} < k^{-1/3} \leq \beta(\hat{\mathbf p}_t),
\qquad \forall t \leq m_T.
$$
Therefore, the stopping condition defining $\tau_2$ cannot be satisfied at any time $t \leq m_T$, and thus
\begin{align}
    \tau_2 > m_T.
    \label{eq:unknown-tau2-lower}
\end{align}

On the other hand, on the event $\mc{G}_1$, we have for every $t \in [m_T,T]$ and every $j \in [k]$,
\begin{align}
    \frac{p_j}{2} \leq \hat p_{t,j} \leq \frac{3p_j}{2}.
    \label{eq:unknown-phat-good}
\end{align}
Since \eqref{eq:unknown-tau2-lower} implies $\tau_2-1 \geq m_T$, it follows from \eqref{eq:unknown-phat-good} that, on $\mc{G}_1$,
\begin{align}
    \frac{p_j}{2} \leq \phfreezej \leq \frac{3p_j}{2},
    \qquad \forall j \in [k].
    \label{eq:unknown-freeze-good}
\end{align}

This implies
$$
2^{-2/3} p_j^{2/3} \leq (\phfreezej)^{2/3} \leq \paran{\frac32}^{2/3} p_j^{2/3}.
$$
Summing over $j$ and raising to the power $3/2$ proves
\begin{align}
    \frac12 \lpnorm{\mathbf p}{2/3}
    \leq
    \lpnorm{\phfreeze}{2/3}
    \leq
    \frac32 \lpnorm{\mathbf p}{2/3}.
    \label{eq:unknown-freeze-norm}
\end{align}

We now analyze phase (III). Let
$$
s \coloneqq \frac{\tau_2-1}{T}.
$$
By the definition of $\tau_2$ and the modified algorithm, we have
$$
s \leq \beta(\phfreeze)=1-\alphmininp{\phfreeze}.
$$
Therefore, by Lemma~\ref{lem:alphamin}, the optimizer of the budgeted problem with passive fraction $s$ and distribution $\phfreeze$ coincides with the fully active optimizer for $\phfreeze$, namely
\begin{align}
    q_j
    =
    \frac{(\phfreezej)^{2/3}}{\sum_{\ell \in [k]} (\hat p ^{\text{fr}}_{\ell})^{2/3}},
    \qquad j \in [k].
    \label{eq:unknown-q}
\end{align}
This is exactly the vector used by the algorithm in the active phase, and the active phase sampling probabilities are then
$$
r_j = \frac{q_j - s\phfreezej}{1-s}.
$$

We next show that the final number of observations of each subpopulation is of order $Tq_j$ with high probability. For $j \in [k], t \in [T]$, let
$$
N_{t,j} \coloneqq \sum_{s=1}^{t} \mb{1}\{C_s=j\}
$$
denote the number of times EETC observes subpopulation $j$ by time $t$. Conditional on $\phfreeze$ and $\tau_2$, the passive counts are fixed and equal to
$$
N_{\tau_2-1,j} = sT \phfreezej.
$$
Let $A_j$ denote the number of times subpopulation $j$ is selected during the active phase. Conditional on $\phfreeze$ and $\tau_2$, the variable $A_j$ is a sum of independent Bernoulli random variables with mean
$$
M_j \coloneqq T(q_j-s\phfreezej).
$$
Define
$$
\mc{J}_{\mathrm{big}}
\coloneqq
\brak{j \in [k] : M_j \geq 24\ln(2kT)},
\qquad
\mc{J}_{\mathrm{small}} \coloneqq [k]\setminus \mc{J}_{\mathrm{big}}.
$$
For each $j \in \mc{J}_{\mathrm{big}}$, Lemma~\ref{lem:independent-bernoulli-count} with $\eta=1/(kT)$ gives
$$
\pr\paran{A_j \geq \frac{M_j}{2} \,\middle|\, \phfreeze,\tau_2} \geq 1-\frac{1}{kT}.
$$
Let
$$
\mc{G}_2 \coloneqq \bigcap_{j \in \mc{J}_{\mathrm{big}}} \brak{A_j \geq \frac{M_j}{2}}.
$$
Then, conditional on $(\phfreeze,\tau_2)$, a union bound gives
$$
\pr(\mc{G}_2^c \mid \phfreeze,\tau_2) \leq \frac{1}{T}.
$$
Taking expectation with respect to $(\phfreeze,\tau_2)$ yields
\begin{align}
    \pr(\mc{G}_2^c)
    =
    \ebb\brak{\pr(\mc{G}_2^c \mid \phfreeze,\tau_2)}
    \leq
    \frac{1}{T}.
    \label{eq:unknown-g2}
\end{align}

We claim that on $\mc{G}_1 \cap \mc{G}_2$,
\begin{align}
    N_{T,j} \geq \frac{Tq_j}{2},
    \qquad \forall j \in [k].
    \label{eq:unknown-final-counts}
\end{align}
Indeed, if $j \in \mc{J}_{\mathrm{big}}$, then on $\mc{G}_2$,
\begin{align*}
    N_{T,j}
    &=
    sT\phfreezej + A_j
    \geq
    sT\phfreezej + \frac{M_j}{2} \\
    &=
    sT\phfreezej + \frac{T(q_j-s\phfreezej)}{2}
    =
    \frac{T(q_j+s\phfreezej)}{2}
    \geq
    \frac{Tq_j}{2}.
\end{align*}
If $j \in \mc{J}_{\mathrm{small}}$, then $M_j < 24\ln(2kT)$. We also have
$$
N_{T,j} \geq sT\phfreezej = Tq_j - M_j.
$$
Moreover, from \eqref{eq:unknown-q},
$$
q_j
=
\frac{(\phfreezej)^{2/3}}{\sum_{\ell \in [k]} (\hat p ^{\text{fr}}_{\ell})^{2/3}}
\geq
\hat p ^{\text{fr}}_{\min},
$$
because
$$
\sum_{\ell \in [k]} (\hat p ^{\text{fr}}_{\ell})^{2/3}
\leq
(\hat p ^{\text{fr}}_{\min}) ^{-1/3} \sum_{\ell \in [k]} \hat p ^{\text{fr}}_{\ell}
=
(\hat p ^{\text{fr}}_{\min})^{-1/3}.
$$
Together with \eqref{eq:unknown-freeze-good}, this implies
$$
q_j \geq \hat p ^{\text{fr}}_{\min} \geq \frac{p_{\min}}{2}.
$$
Since $T \gg 1/p_{\min}, k$, we have
$$
Tq_j \geq \frac{Tp_{\min}}{2} \geq 48\ln(2kT),
$$
and therefore, for $j \in \mc{J}_{\mathrm{small}}$,
$$
N_{T,j} \geq Tq_j - M_j \geq Tq_j - 24\ln(2kT) \geq \frac{Tq_j}{2}.
$$
This proves \eqref{eq:unknown-final-counts}.

Now fix an arbitrary instance. Let $\Delta_j$ denote the recommendation gap in subpopulation $j$ at the end of the horizon. Since the bandit subroutine $\algsr$ is anytime and achieves the optimal standard bandit simple regret rate, there exists a universal constant $c_3>0$ such that, conditional on $N_{T,j}$,
$$
\ebb[\Delta_j \mid N_{T,j}] \leq c_3 \sqrt{\frac{n}{N_{T,j}}}.
$$
Therefore,
\begin{align*}
    \text{SR}(\texttt{EETC},T,\mathbf p)
    &=
    \sum_{j \in [k]} p_j \ebb[\Delta_j] \\
    &=
    \sum_{j \in [k]} p_j \ebb[\Delta_j \mb{1}_{\mc{G}_1 \cap \mc{G}_2}]
    +
    \sum_{j \in [k]} p_j \ebb[\Delta_j \mb{1}_{(\mc{G}_1 \cap \mc{G}_2)^c}] \\
    &\leq
    c_3 \sum_{j \in [k]} p_j \ebb\brak{\sqrt{\frac{n}{N_{T,j}}}\mb{1}_{\mc{G}_1 \cap \mc{G}_2}}
    + \pr((\mc{G}_1 \cap \mc{G}_2)^c).
\end{align*}
Using \eqref{eq:unknown-final-counts}, this yields
\begin{align}
    \text{SR}(\texttt{EETC},T,\mathbf p)
    \leq
    c_3 \sqrt{\frac{2n}{T}} \sum_{j \in [k]} \frac{p_j}{\sqrt{q_j}}
    + \pr((\mc{G}_1 \cap \mc{G}_2)^c).
    \label{eq:unknown-regret-pre}
\end{align}

It remains to bound the sum. By \eqref{eq:unknown-q},
\begin{align*}
    \sum_{j \in [k]} \frac{p_j}{\sqrt{q_j}}
    &=
    \paran{\sum_{\ell \in [k]} (\hat p ^{\text{fr}}_{\ell})^{2/3}}^{1/2}
    \sum_{j \in [k]} p_j (\phfreezej)^{-1/3}.
\end{align*}
From \eqref{eq:unknown-freeze-good},
$$
\sum_{\ell \in [k]} (\hat p ^{\text{fr}}_{\ell})^{2/3}
\leq
\paran{\frac32}^{2/3} \sum_{\ell \in [k]} p_\ell^{2/3},
\qquad
(\phfreezej)^{-1/3} \leq 2^{1/3} p_j^{-1/3}.
$$
Therefore,
\begin{align}
    \sum_{j \in [k]} \frac{p_j}{\sqrt{q_j}}
    &\leq
    3^{1/3}
    \paran{\sum_{\ell \in [k]} p_\ell^{2/3}}^{1/2}
    \sum_{j \in [k]} p_j^{2/3}
    \nonumber\\
    &=
    3^{1/3} \lpnorm{\mathbf p}{2/3}.
    \label{eq:unknown-s-bound}
\end{align}

Finally, by \eqref{eq:unknown-g1} and \eqref{eq:unknown-g2},
$$
\pr((\mc{G}_1 \cap \mc{G}_2)^c) \leq \pr(\mc{G}_1^c)+\pr(\mc{G}_2^c) \leq \frac{2}{T}.
$$
Substituting this and \eqref{eq:unknown-s-bound} into \eqref{eq:unknown-regret-pre}, we obtain
$$
\text{SR}(\texttt{EETC},T,\mathbf p)
\leq
c_4 \sqrt{\frac{n}{T}} \, \lpnorm{\mathbf p}{2/3}
+
\frac{2}{T}
$$
for a universal constant $c_4>0$. Hence, for sufficiently large $T$,
$$
{\text{SR}(\texttt{EETC}, T, \mathbf p)}
\in
O\Big(\sqrt{\frac{n}{T}} \, \lpnorm{\mathbf p}{2/3}\Big).
$$

The final claim follows from the known-$\mathbf p$ analysis: the optimal fully active policy that knows $\mathbf p$ achieves rate
$$
\Theta\Big(\sqrt{\frac{n}{T}} \, \lpnorm{\mathbf p}{2/3}\Big),
$$
and this also matches the known-$\mathbf p$ lower bound up to constant factors. Therefore EETC achieves the same simple regret rate, up to constants, as the optimal fully active known-$\mathbf p$ policy.
\end{proof}

\section{Additional Experiments}\label{apx:exp}

This section provides additional details for the real-world experiment on the MovieLens dataset described in the main text, together with further experimental results.

\subsection{Details of the MovieLens Dataset}

Our real-world experiment is based on the MovieLens 1M dataset, which contains $1{,}000{,}209$ ratings from $6{,}040$ users on roughly $3{,}900$ movies. We construct a bandit instance by restricting attention to the five most frequently rated movies in the dataset, namely \emph{American Beauty (1999)}, \emph{Star Wars: Episode IV - A New Hope (1977)}, \emph{Star Wars: Episode V - The Empire Strikes Back (1980)}, \emph{Star Wars: Episode VI - Return of the Jedi (1983)}, and \emph{Jurassic Park (1993)}. These five movies define the treatment set, so that $n=5$.

The subpopulations are defined by the Cartesian product of gender and age group, yielding $2 \times 7 = 14$ demographic groups and therefore $k=14$ subpopulations. After filtering the data to the selected movies and joining with the demographic attributes, the resulting dataset contains $14{,}964$ rating records. Each record consists of a treatment--subpopulation pair together with the corresponding MovieLens rating, which we use directly as the reward. Thus, rewards take values in $\{1,2,3,4,5\}$.

For each subpopulation $j$, the population weight $p_j$ is estimated by the empirical fraction of filtered records belonging to that group. The resulting subpopulation distribution is
\begin{align*}
    \mathbf p =
(0.0084,\;0.0444,\;0.0863,\;0.0456,\;0.0216,\;0.0165,\;0.0088, \\ 
\;0.0233,\;0.1549,\;0.2994,\;0.1474,\;0.0558,\;0.0539,\;0.0336).
\end{align*}

\begin{figure}[t]
    \centering
    \begin{subfigure}[b]{0.48\textwidth}
        \centering
        \includegraphics[width=\textwidth]{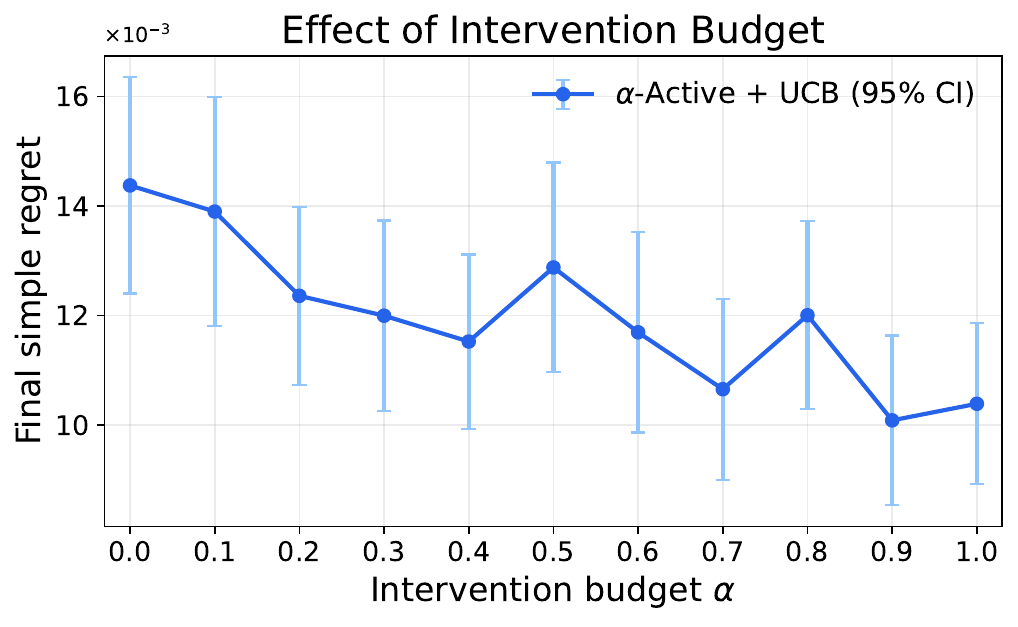}
        \caption{$T=5000$}
        \label{fig:budget_movielens_5000}
    \end{subfigure}
    \hfill
    \begin{subfigure}[b]{0.48\textwidth}
        \centering
        \includegraphics[width=\textwidth]{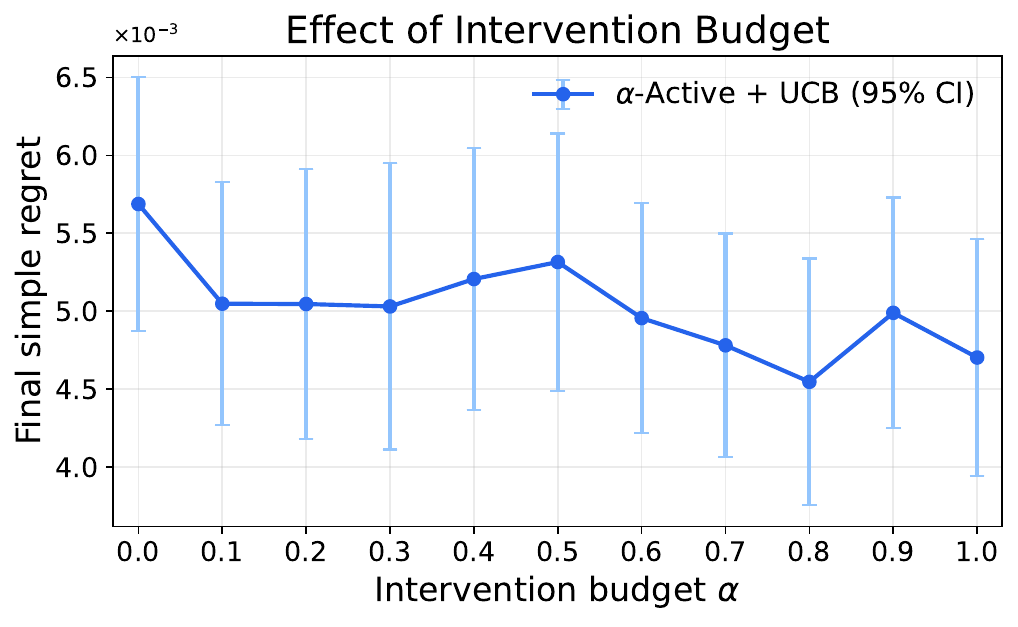}
        \caption{$T=10000$}
        \label{fig:budget_movielens_10000}
    \end{subfigure}
    \caption{Simple regret of the budgeted active algorithm on the MovieLens instance as a function of the intervention budget $\alpha$. As predicted by the theory, the regret decreases as the active budget increases, interpolating between the passive and fully active regimes.}
    \label{fig:budget_movielens}
\end{figure}

The coordinates are ordered by gender and age group, with the seven female groups first and the seven male groups second:
\begin{align*}
    &(F,\text{ Under 18}),\; (F,\text{ 18--24}),\; (F,\text{ 25--34}),\; (F,\text{ 35--44}),\; (F,\text{ 45--49}),\; (F,\text{ 50--55}),\; (F,\text{ 56+}), \\
    &(M,\text{ Under 18}),\; (M,\text{ 18--24}),\; (M,\text{ 25--34}),\; (M,\text{ 35--44}),\; (M,\text{ 45--49}),\; (M,\text{ 50--55}),\; (M,\text{ 56+}).
\end{align*}

Similarly, for each treatment--subpopulation pair $(a,j)$, the mean reward $\mu_{a,j}$ is computed as the empirical average of the corresponding ratings in the filtered dataset. During interaction, when an algorithm selects treatment $a$ for subpopulation $j$, the environment samples uniformly with replacement from the historical ratings associated with $(a,j)$ and returns the sampled rating as the reward.

At the end of the horizon, each algorithm recommends one treatment for every subpopulation, and performance is evaluated using the context-weighted simple regret defined in the main text, with the empirical MovieLens frequencies and empirical cell means used as the ground-truth values.

\subsection{Budgeted Interventions on MovieLens}

We next evaluate the budgeted active algorithm on the MovieLens instance \cite{harper2015movielens}. In this experiment, we vary the intervention budget $\alpha$ from $0$ to $1$ in increments of $0.1$, where $\alpha=0$ corresponds to the fully passive policy and $\alpha=1$ to the fully active policy. We report results for two horizons, $T=5000$ and $T=10000$, and for UCB as the standard bandit subroutine.

Figure~\ref{fig:budget_movielens} shows the average simple regret over $100$ independent runs, together with $95\%$ confidence intervals. Although the curves are somewhat noisy due to the variance of the real data environment, both figures clearly exhibit the expected decreasing trend: as the intervention budget $\alpha$ increases, the simple regret decreases. This is consistent with the theory, which predicts that allowing more active control over the sampled subpopulations improves the achievable regret by moving the allocation closer to the optimal active design.

Overall, these results support the budgeted intervention analysis and illustrate the gradual interpolation between passive and fully active performance on real data.

\subsection{Synthetic Experiments}

We next report synthetic experiments designed to highlight the gains predicted by the theory. In all synthetic experiments, for each treatment--subpopulation pair $(a,j)$, the reward distribution is Bernoulli and independent across pairs. For each subpopulation $j$, one treatment is selected uniformly at random to be the optimal one, and its mean is set to
$$
\frac{1}{2} + \delta_j,
\qquad
\delta_j \coloneqq \sqrt{\frac{n}{T}}\, p_j^{-1/3},
$$
while the remaining $n-1$ treatments have mean $1/2$. This construction follows the worst-case instances appearing in our theoretical analysis.

To induce a non-uniform subpopulation distribution with a pronounced active-passive gap, we choose $\mathbf p$ as
$$
\paran{1 - \epsilon,\; \frac{\epsilon}{{k-1}},\; \dots,\; \frac{\epsilon}{{k-1}} },
$$
where $\epsilon = \frac{1}{\sqrt{k-1}}$.
That is, one subpopulation has a larger weight, while the remaining $k-1$ subpopulations have equal smaller weights. This choice is aligned with the worst-case regime discussed in the proof of Lemma \ref{lem:active_passive_gap}, in which the gap factor $R(\mathbf p)$ becomes large. All results below are averaged over $100$ independent runs and are reported with $95\%$ confidence intervals.

We compare three algorithms: the passive baseline, the known-$p$ fully active algorithm, and EETC. In each case, the standard bandit subroutine is UCB, since it has a better performance compared to the uniform subroutine.

\paragraph{Varying the number of subpopulations.}
In the first synthetic experiment, we fix $n=5$ and vary the number of subpopulations over
$
k \in \{5,10,20,30,40\}.
$
For each value of $k$, we set
$
T = 100nk,
$
so that the horizon scales linearly with the total number of treatment--subpopulation pairs and the algorithms have sufficient budget to explore all pairs.

Figure~\ref{fig:synth-vary-k} summarizes the results. The left panel reports the average simple regret, while the right panel reports the regret ratio between the passive baseline and each algorithm, namely
$$
\frac{\text{Regret of passive baseline}}{\text{Regret of algorithm}},
$$
so values larger than $1$ indicate improvement over the passive policy. The right panel also includes the theoretical quantity $R(\mathbf p)$.

The figure shows that both proposed algorithms uniformly outperform the passive baseline over all tested values of $k$. Moreover, the gap becomes larger as $k$ increases, which is consistent with the theory, since the value of $R(\mathbf p)$ also increases with $k$ in this construction. The known-$p$ fully active policy and EETC are not clearly separated from one another, which is expected: they have the same asymptotic rate, and the finite sample variance is relatively high.

\begin{figure}[t]
    \centering
    \begin{subfigure}[b]{0.48\textwidth}
        \centering
        \includegraphics[width=\textwidth]{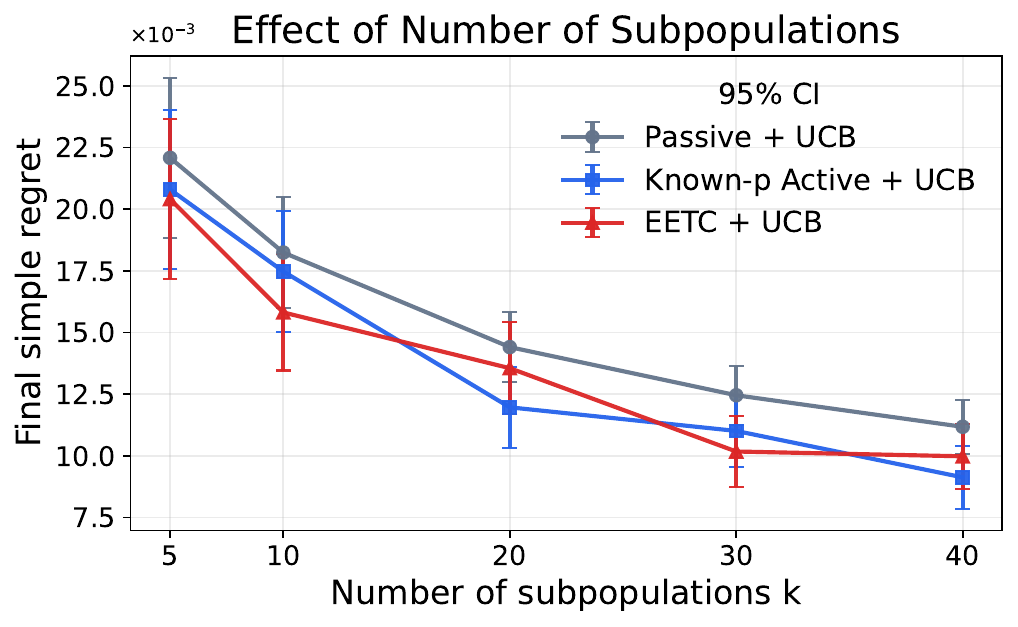}
        \caption{Average simple regret.}
        \label{fig:synth-vary-k-regret}
    \end{subfigure}
    \hfill
    \begin{subfigure}[b]{0.48\textwidth}
        \centering
        \includegraphics[width=\textwidth]{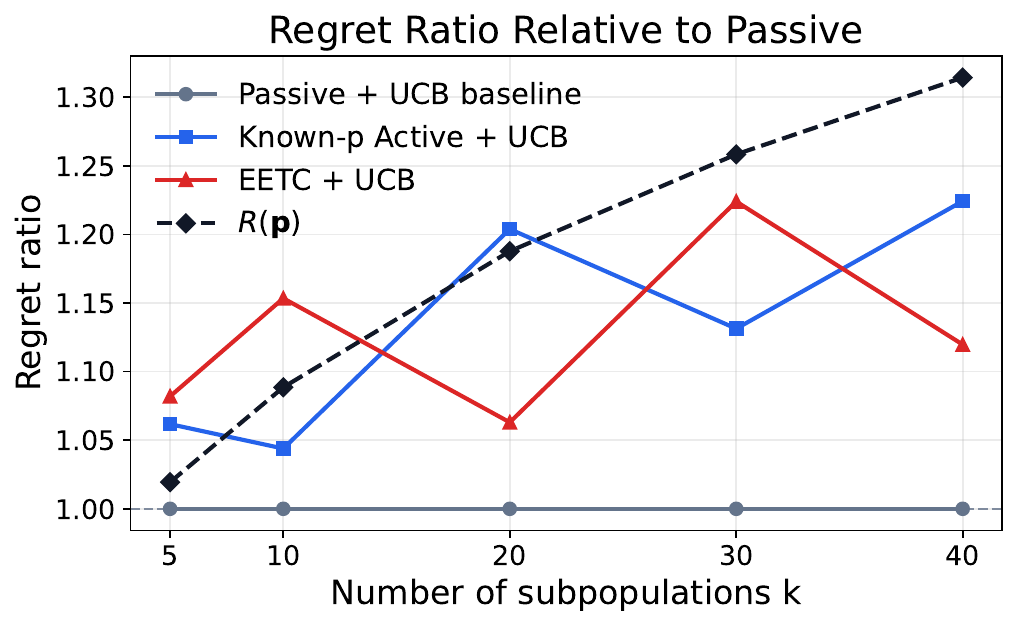}
        \caption{Improvement ratio over passive baseline.}
        \label{fig:synth-vary-k-ratio}
    \end{subfigure}
    \caption{Synthetic experiment with varying number of subpopulations $k$ and fixed $n=5$, with horizon $T=100nk$. Left: average simple regret over $100$ runs with $95\%$ confidence intervals. Right: ratio between the average regret of the passive baseline and the regret of each algorithm. Values above $1$ indicate improvement over the passive policy. The theoretical factor $R(\mathbf p)$ is also shown for reference.}
    \label{fig:synth-vary-k}
\end{figure}

\begin{figure}[t]
    \centering
    \begin{subfigure}[b]{0.48\textwidth}
        \centering
        \includegraphics[width=\textwidth]{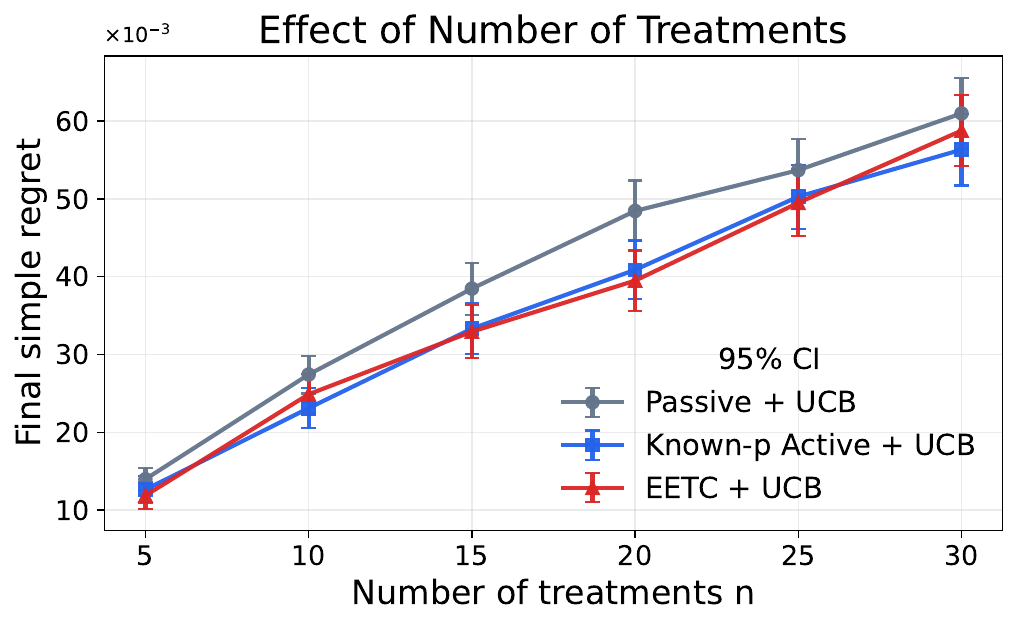}
        \caption{Average simple regret.}
        \label{fig:synth-vary-n-regret}
    \end{subfigure}
    \hfill
    \begin{subfigure}[b]{0.48\textwidth}
        \centering
        \includegraphics[width=\textwidth]{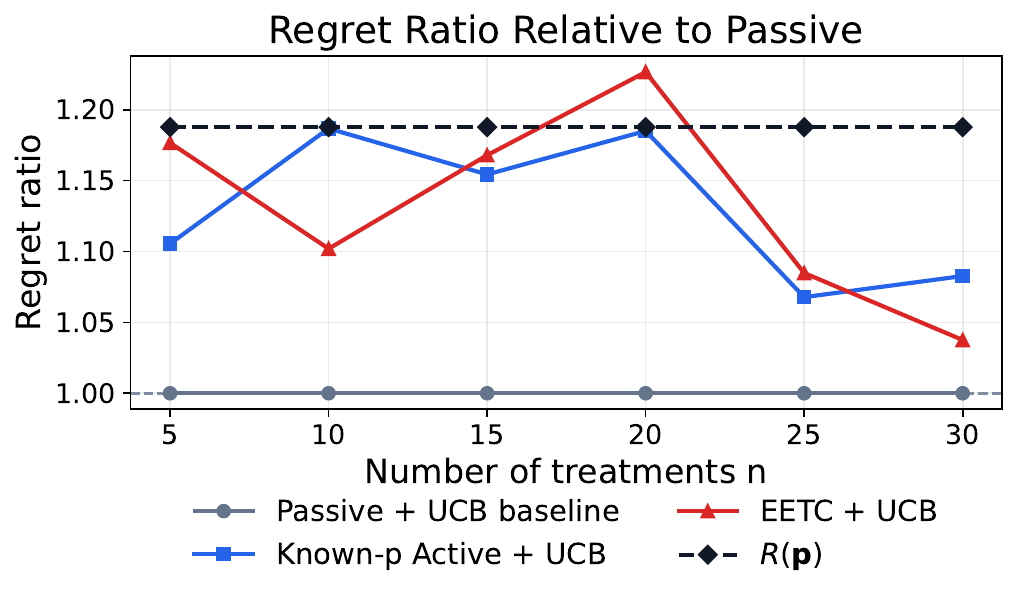}
        \caption{Improvement ratio over passive baseline.}
        \label{fig:synth-vary-n-ratio}
    \end{subfigure}
    \caption{Synthetic experiment with varying number of treatments $n$ and fixed $k=20$, with horizon $T=10000$. Left: average simple regret over $100$ runs with $95\%$ confidence intervals. Right: ratio between the average regret of the passive baseline and the regret of each algorithm. Values above $1$ indicate improvement over passive sampling, while the passive baseline appears as the constant line at value $1$.}
    \label{fig:synth-vary-n}
\end{figure}

\paragraph{Varying the number of treatments.}
In the second synthetic experiment, we fix the number of subpopulations to $k=20$ and vary the number of treatments over
$$
n \in \{5,10,15,20,25,30\}.
$$
Unlike the previous experiment, we keep the horizon fixed at $T=10000$ for all values of $n$. For each value of $n$, we fix one synthetic instance and average the final simple regret over $100$ independent runs, varying only the reward samples and the algorithmic randomness. All algorithms use UCB as the within-subpopulation subroutine.

Figure~\ref{fig:synth-vary-n} reports the same two summaries as in the previous experiment. The left panel shows the average simple regret with $95\%$ confidence intervals, and the right panel shows the regret ratio relative to the passive baseline,
$$
\frac{\text{Regret of passive baseline}}{\text{Regret of algorithm}}.
$$
Thus, values above $1$ indicate improvement over passive sampling, while the passive baseline itself appears as the constant line at value $1$.

The active algorithms outperform the passive baseline for all tested values of $n$. The relative gain becomes smaller as $n$ increases, which is expected in this fixed-horizon experiment: larger $n$ creates more treatment--subpopulation reward distributions to explore, while the total budget remains unchanged. As before, the known-$p$ active algorithm and EETC are close to each other and are not uniformly ordered, consistent with their matching theoretical rates and the finite-sample variability of the experiment.





\end{document}